\documentclass[letterpaper]{article} 
\usepackage[preprint]{aaai2027}  
\usepackage[hyphens]{url}  
\usepackage{graphicx}
\usepackage{natbib} 
\usepackage{caption} 
\usepackage{amsmath}
\usepackage{amssymb}
\usepackage{amsthm}
\usepackage{multirow}
\usepackage{wasysym} 
\usepackage{wasysym}

\usepackage[table]{xcolor}
\definecolor{familybg}{rgb}{0.91, 0.93, 0.96}

\usepackage{algorithm}
\usepackage{algorithmic}
\usepackage{booktabs}
\usepackage{amsthm}

\newtheorem{theorem}{Theorem}

\newtheorem{proposition}{Proposition}

\newtheorem{remark}{Remark}
\newcommand{\invkd}{\textsc{inv\_kd}}

\newtheorem{definition}{Definition}

\usepackage{newfloat}
\usepackage{listings}
\DeclareCaptionStyle{ruled}{labelfont=normalfont,labelsep=colon,strut=off} 
\floatstyle{ruled}
\newfloat{listing}{tb}{lst}{}
\floatname{listing}{Listing}

\usepackage{booktabs}
\usepackage{microtype}

\title{Do Student LLMs Inherit OOD Robustness? \\ Invariance-Weighted Distillation for Reliable Knowledge Transfer}

\author{
    Dileesha Kannangara,
    Sanghamitra Dutta
}

\affiliations{
    University of Maryland, College Park\\
    dileesha@umd.edu, sanghamd@umd.edu
}

\begin{document}

\maketitle

\begin{abstract}

Knowledge distillation (KD) aims to compress high-performance teacher LLMs into lightweight students. However, distilled students often exhibit substantial performance degradation in out-of-distribution (OOD) settings, a critical gap that remains underexplored. We identify two compounding mechanisms causing OOD performance degradation: (1) data spuriousness: students can learn spurious correlations in the distillation dataset over genuine causal relationships; and (2) teacher capability: standard KD treats all samples uniformly, ignoring whether the teacher is guided by causal features or misled by spurious shortcuts on a given sample. To address these challenges, we propose Invariance-Weighted Distillation (IWD), a theoretically grounded framework that dynamically reweights training samples using an estimate of the teacher's causal reliance derived from prediction invariance across multiple synthetic environments. IWD perturbs spurious cues while preserving core semantics, assigning higher distillation weights to samples whose teacher predictions remain invariant, indicating greater reliance on causal rather than spurious features. We theoretically show that IWD reduces the student's Spurious-to-Causal (S2C) gradient ratio compared to standard uniformly weighted KD, driving the student toward more invariant representations. Experiments on four NLP benchmarks (MNLI, SQuAD-v2, CoNLL-2003 NER, and SST-2) across two model families (DeBERTa-v3 and Qwen-2.5) demonstrate that IWD consistently outperforms strong KD baselines on OOD evaluations while maintaining competitive in-distribution (ID) performance. Specifically, IWD achieves the highest accuracy in 15 out of 16 OOD benchmarks and improves average OOD performance over standard KD by 4.34 percentage points on NLI and 14.94 percentage points on QA.

\end{abstract}

\section{Introduction}

Large Language Models (LLMs) are central to modern applications due to their remarkable ability to solve complex tasks~\citep{openai2024gpt4,yang2025qwen25,deepseek2025r1}. However, high computational and memory demands hinder training, fine-tuning, and deployment in resource-constrained environments~\citep{girija2025optimizing}. Knowledge distillation (KD) addresses these challenges by training a compact student model to mimic the predictive behavior of a high-capacity teacher model across a designated distillation dataset~\citep{hinton2015distilling,xu2024surveykd}. Distilled students deliver competitive performance with drastically reduced resource overhead, enabling efficient real-world deployment.

However, student models often demonstrate significant performance degradation on out-of-distribution (OOD) data compared to the in-distribution (ID) data used during training~\citep{du2023robustness, stanton2021does}. OOD data can be formally defined as data originating from a distribution that differs from that encountered during any stage of training prior to deployment~\citep{shen2021towards}. Although teacher models themselves are not fully robust to OOD data, we observe that the performance gap between teacher and student models becomes substantially larger under OOD settings than under ID settings. This suggests that current KD techniques do not effectively transfer the teacher model's existing ability to handle OOD data and directly undermine this capability in the resulting student models~\citep{stacey2024distilling}.
This leads to two central questions: \textit{To what extent do student LLMs inherit OOD robustness from their teachers, and how can we transform distillation to enable reliable knowledge transfer under distribution shifts?}

We identify two main limitations of current KD techniques that cause OOD performance degradation. First, \textbf{under existing distillation methods, students learn spurious correlations from the distillation dataset instead of genuine causal relationships}. Standard KD approaches train the student model to mimic the teacher's outputs~\citep{hinton2015distilling, gu2023minillm} or intermediate representations~\citep{romero2015fitnets, sun2019pkd, sun2020mobilebert, jiao2020tinybert} by minimizing discrepancies over the distillation dataset. Additionally, distillation minimizes the standard cross-entropy loss over these training samples. In doing so, the student learns dataset-specific spurious correlations, reducing its ability to generalize when encountering OOD data. Second, \textbf{standard distillation treats all training samples uniformly, ignoring the teacher's ability to isolate causal signals from spurious correlations across individual samples}. In practice, a teacher model may successfully resist spurious shortcuts on samples with strong causal semantics, yet succumb to spurious correlations on others. Standard KD applies a uniform weight across all samples, inadvertently forcing the student to imitate the teacher even when the teacher relies on brittle, spurious associations. For reliable knowledge transfer, distillation should prioritize samples where the teacher exhibits high fidelity to invariant causal features while suppressing spurious influences.

\begin{figure*}[t]
\centering
\includegraphics[width=0.95\textwidth]{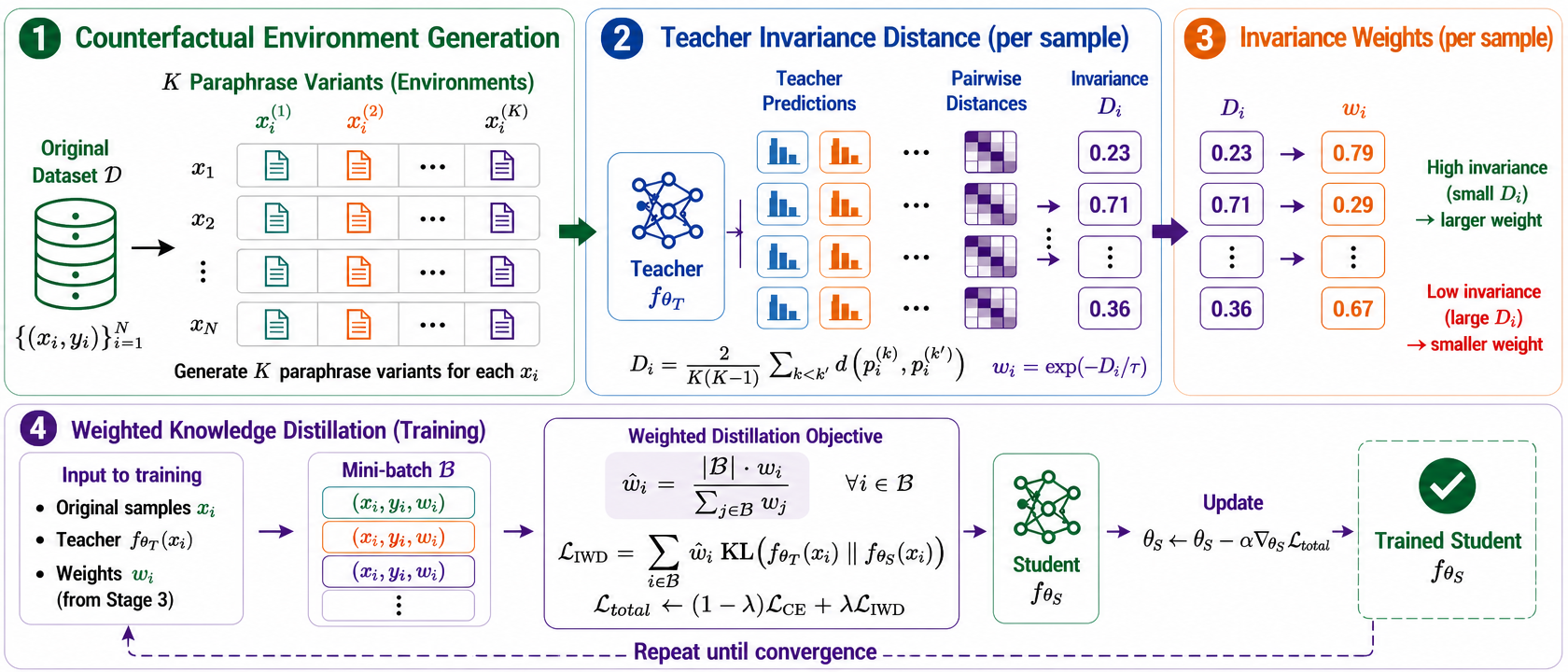}
\caption{\textbf{Overview of IWD.} The pipeline operates in four stages: 
\textbf{(1)~Counterfactual Environment Generation:} Generating semantically equivalent paraphrase variants; 
\textbf{(2)~Teacher Invariance Distance:} Measuring prediction consistency across variants using the teacher predictions; 
\textbf{(3)~Weight Calculation:} Assigning higher reliability weights to samples where teacher predictions remain invariant; and 
\textbf{(4)~Weighted Knowledge Distillation:} Training the student model using the weighted loss objective.}
\label{fig:overview}
\end{figure*}

To address these limitations, we propose \textbf{Invariance-Weighted Distillation (IWD)}, a novel strategy that uses a weighted objective function to prioritize samples where the teacher relies on causal relationships over spurious correlations (see Fig.~\ref{fig:overview}). We first generate multiple synthetic environments for each training sample by altering its spurious features while preserving its original semantic meaning and label. This environment generation step resolves the first limitation by exposing the student to diverse spurious perturbations across counterfactual variants with fixed core semantics. Distilling across these environments prevents the student from learning dataset-specific spurious correlations and forces it to learn invariant representations that generalize across distribution shifts. To address the second limitation, we measure the discrepancy among the teacher's output distributions across these environments. This distance quantifies the teacher's prediction invariance, i.e., how strongly the teacher relies on invariant features versus environment-specific cues. We incorporate this distance into the distillation loss as a weight, increasing emphasis on samples where the teacher exhibits high prediction invariance and, consequently, high fidelity to invariant causal features.

\paragraph{Contributions}Our main contributions are as follows:
\begin{itemize}
    \item \textbf{Conceptualizing OOD Vulnerabilities in KD.} We identify two key failure modes of standard KD under distribution shifts: (1) student models learn spurious correlations from the distillation dataset, and (2) current KD methods treat all training samples equally, failing to recognize the teacher's ability to distinguish between causally robust and spurious-dominated samples.
    
    \item \textbf{Invariance-Weighted Distillation (IWD).} We propose IWD, a novel distillation strategy that assigns sample-specific weights based on the teacher's prediction invariance across synthetically perturbed environments (see Fig.~\ref{fig:overview}). To overcome Limitation 1, IWD varies spurious features while preserving core causal semantics, exposing the student to diverse counterfactual variations of each sample. To address Limitation 2, IWD prioritizes samples where the teacher exhibits high invariance, and thus high fidelity to causal features. Our strategy equips the student with improved robustness to spurious features over standard KD and conventional data augmentation.

    \item \textbf{Theoretical Guarantee.} We provide a theoretical guarantee for IWD by leveraging a metric called the Spurious-to-Causal (S2C) gradient ratio (see Theorem~\ref{thm:s2c}). We prove that IWD yields a lower S2C ratio than standard distillation, guaranteeing that the student's expected gradient aligns more strongly with invariant  signals while reducing spurious feature transfer.
    
    \item \textbf{Empirical Validation.} We evaluate IWD across four distinct NLP tasks (NLI, extractive QA, NER, and sentiment classification) and two model families (DeBERTa-v3 and Qwen-2.5), using several OOD datasets per task under various subsampling regimes. IWD consistently improves OOD robustness while maintaining ID performance. For instance, across model families, IWD improves average OOD performance over standard KD by 4.34 and 14.94 percentage points on NLI and QA, respectively.

\end{itemize}

\paragraph{Related Works} Existing KD methods~\citep{hinton2015distilling,zhu2024surveycompression,xu2024surveykd,romero2015fitnets,sun2019pkd,zagoruyko2017attention,tian2020contrastive,kim2016sequence,hamman2025few,pmlr-v258-dissanayake25a} focuses on matching teacher and student outputs, intermediate features, or attention patterns. However, the robustness of distilled models under distribution shifts remains underexplored. While ShiftKD~\citep{zhang2025shiftkd} shows that distillation suffers substantial performance degradation under covariate and correlation shifts, other approaches attempt to enhance robustness via confidence-guided data augmentation~\citep{popp2025confidence}, adversarial distillation~\citep{goldblum2019ard}, or distillation from robust foundation models~\citep{zhou2023dad}. Yet, these methods primarily target conventional deep learning models, leaving OOD robustness in distilled LLMs underexplored.

Prior work on OOD robustness in distilled LLMs seeks to improve performance through data augmentation, sample selection, or adaptive distillation strategies.\citep{stacey2024distilling} use domain-targeted data augmentation and sampling strategies to improve NLI distillation under distribution shift. Prior works \citet{chen-etal-2023-disco} and \citet{hamman2025few} leverage LLMs to generate label-flipping counterfactuals, prompting students to learn causal features over spurious correlations. Other works introduce difficulty-aware distillation using model pruning~\citep{du2023robustness} or OOD-guided data generation to expose students to more diverse samples~\citep{gholami2024gold}. Notably, these works do not explicitly enforce invariance across semantically equivalent perturbations and rely on uniform distillation objectives that treat all training samples equally. \emph{In contrast, our approach evaluates the teacher's prediction invariance across synthetically perturbed environments with varying spurious features, leveraging the teacher-invariance signal to weight the distillation loss toward more reliable samples.}

\section{Preliminaries}

Knowledge distillation~\citep{xu2024surveykd,hamman2025few} trains a smaller student $f_{\theta_S}$ to imitate a larger teacher $f_{\theta_T}$. For an input $x$ with target label $y$, the student's loss is:
\begin{equation}
\label{eq:kd}
\mathcal{L}_{\text{KD}} {=} (1{-}\lambda)\,\mathcal{L}_{\text{CE}}(y, p_S(x)) {+} \lambda\,\mathcal{L}_{\text{KL}}\,(p_T(x) \parallel p_S(x)),
\end{equation}
Here, $\mathcal{L}_{\text{CE}}$ is the cross-entropy (CE) loss, $\mathcal{L}_{\text{KL}}$ is the KL divergence~\citep{kullback1951} at each generation step over the vocabulary, $p_{\text{S}}(x) = \mathrm{softmax}(f_{\theta_{\text{S}}}(x))$ and $p_{\text{T}}(x) = \mathrm{softmax}(f_{\theta_{\text{T}}}(x))$ are the output distributions, and $\lambda \in [0,1]$ balances hard-label CE and soft target imitation.

Out-of-distribution (OOD) refers to the setting in which a deployed model make predictions on inputs that come from a different distribution from the one it was trained on. A model is said to be OOD-robust if its performance does not degrade substantially under such a shift. Formally, let $P_{\text{train}}$ denote the training distribution and $P_{\text{OOD}} \neq P_{\text{train}}$ denote a shifted evaluation distribution. Standard supervised learning offers no formal guarantee in this regime, and trained models often degrade sharply~\citep{hendrycks2019benchmarking,geirhos2020shortcut}, particularly when they exploit features that are spuriously correlated with the label in $P_{\text{train}}$ but uncorrelated in $P_{\text{OOD}}$. OOD robustness is measured via the metric:
\begin{equation} \label{eq:phi}
\phi(\theta) \;=\;
\mathbb{E}_{(x,y)\sim P_{\text{OOD}}}\!\left[\,
\mathbf{1}[f_\theta(x) = y]\,\right],
\end{equation}
where $\theta$ denotes the parameters of the model under evaluation, $f_\theta(x)$ is its predicted label for input $x$, and $\mathbf{1}[\cdot]$ is the indicator function returning $1$ when its argument is true and $0$ otherwise. The expectation $\mathbb{E}_{(x,y)\sim P_{\text{OOD}}}[\cdot]$ averages over input-label pairs drawn from the shifted distribution, so $\phi(\theta)$ is the expected accuracy of $f_\theta$ on $P_{\text{OOD}}$.

\noindent \textbf{Problem Setup.}
Our goal is to study how KD in LLMs affects the OOD robustness of the student model. Formally, let $\mathcal{D} = \{(x_i, y_i)\}_{i=1}^{N}$ be a labeled training set drawn from an in-distribution (ID) source $P_{\text{train}}$, and let $P_{\text{OOD}}$ represent a shifted evaluation distribution. Given a teacher model $f_{\theta_{\text{T}}}$ pre-trained on $P_{\text{train}}$, our objective is to learn a smaller student model $f_{\theta_{\text{S}}}$ that maximizes the OOD evaluation metric $\phi(\theta_{\text{S}})$.

Specifically, our goal is to distill student parameters $\theta_S$ that learn genuine causal relationships while suppressing spurious correlations, and achieve higher OOD performance than standard distillation, i.e., $\phi(\theta_S) > \phi(\theta_S^{\text{KD}})$, where $\theta_S^{\text{KD}}$ denotes parameters learned via standard KD.

\section{Main Contributions}
\label{sec:method}
We introduce \textbf{Invariance-Weighted Distillation (IWD)}, a framework enhancing OOD robustness in distilled LLMs. Our strategy operates in four stages: (i)~generation of counterfactual environment variants that vary spurious features while preserving the semantic meaning; (ii)~computation of the teacher-invariance distance that quantifies the teacher's prediction variation across those environments; (iii)~mapping distance to sample weights; and (iv)~distilling using our proposed weighted distillation objective.

\textbf{Stage 1: Counterfactual Environment Generation.} 
For each training sample $(x_i, y_i) \in \mathcal{D}$, we construct a set of $K$ counterfactual variants:
\begin{equation}
\label{eq:env-set}
\mathcal{E}_i \;=\; \bigl\{\, x_i^{(0)},\, x_i^{(1)},\, \dots,\, x_i^{(K-1)} \,\bigr\},
\end{equation}
where $x_i^{(0)} \triangleq x_i$ is the original sample, and $x_i^{(1)}, \dots, x_i^{(K-1)}$ are paraphrases generated across different environments $e^{(k)}$. These environments are designed to perturb surface-level spurious features such as register, syntactic structure, domain vocabulary, capitalization, and sentence length, while holding the core semantic features determining $y_i$ fixed.

Formally, each input can be decomposed into an invariant \emph{causal} component $\mathbf{c}_i$ that determines the label and an environment-specific \emph{spurious} component $\mathbf{s}_i^{(k)}$. The generator ensures these variants satisfy counterfactual label invariance:
\begin{equation}
\label{eq:env-invariance}
P\big(y_i \,\big|\, \mathbf{c}_i,\, \mathbf{s}_i^{(k)}\big)
=
P\big(y_i \,\big|\, \mathbf{c}_i\big)
\quad \forall\, k \in \{0, 1, \dots, K-1\}.
\end{equation}
Equation~\ref{eq:env-invariance} guarantees that the label distribution conditioned on causal features remains identical across all environments, isolating variation strictly to the spurious components. We construct $\mathcal{E}_i$ offline by prompting an auxiliary language model (GPT-4.1-mini) to generate specific spurious perturbations. To enforce label invariance, any paraphrase that would alter the target label $y_i$ is rejected. The union of these environment sets forms the augmented training dataset $\mathcal{D}_{\mathrm{aug}} = \bigcup_{i=1}^{N} \mathcal{E}_i$, which is precomputed once and shared across all experiments.

Our semantics-preserving counterfactual environment generation strategy directly addresses the first limitation by preventing the student from learning dataset-specific spurious correlations. Exposing the model to these variants influences both components of the training objective:
(i)~distillation loss ($\mathcal{L}_{\text{KL}}$), which forces the student to match teacher predictions across varied spurious contexts rather than context-dependent noise; and
(ii)~task loss ($\mathcal{L}_{\text{CE}}$), which trains the student directly on perturbed inputs to ensure ground-truth predictions rely on invariant semantics rather than surface shortcuts.

\paragraph{Stage 2: Teacher Invariance Distance.} Using the environment set $\mathcal{E}_i$, we directly quantify the invariance of the teacher's predictions across counterfactual perturbations.

\begin{definition}[Teacher Invariance Distance]
\label{def:invariance_distance}
Let $\hat{y}_i^{(k)} = f_{\theta_{\text{T}}}\big(x_i^{(k)}\big)$ denote the teacher's prediction on the $k$-th variant in a task-appropriate output space. The teacher invariance distance $D_i$ is defined as the mean pairwise disagreement across all environment pairs:
\begin{equation}
\label{eq:D_i}
D_i \;\triangleq\; \binom{K}{2}^{-1} \sum_{0 \le u < v \le K-1} d \big(\hat{y}_i^{(u)},\, \hat{y}_i^{(v)}\big).
\end{equation}
Here, $d(\cdot, \cdot)$ is a non-negative, symmetric distance metric.
\end{definition}

Since causal features $\mathbf{c}_i$ remain fixed across variants, $D_i$ directly isolates the teacher's sensitivity to spurious perturbations $\mathbf{s}_i^{(k)}$. A small $D_i$ indicates that predictions are invariant and grounded in causal semantics, making the sample a reliable source for distillation. Conversely, a large $D_i$ reveals that predictions fluctuate with surface cues, signaling the teacher's reliance on spurious features that will not generalize in OOD settings. The distance metric $d(\cdot, \cdot)$ is tailored to the output space of the task. For classification (e.g., SST-2, MNLI), $d$ measures distributional divergence over softmax outputs; for structured symbolic outputs like QA text spans or NER entities (SQuAD-v2, CoNLL-2003), $d$ computes per-sample $F_1$ distance ($1 - F_1$) on decoded string predictions.

\paragraph{Stage 3: Weight Calculation.} To convert the invariance distance $D_i$ into a per-sample reliability weight, we apply a monotonically decreasing exponential function:
$w_i \;=\; \exp(-D_i / \tau).$
Here, $\tau > 0$ is a temperature hyperparameter. A small $\tau$ heavily penalizes teacher variance, while a large $\tau$ approaches uniform weighting, recovering environment-augmented distillation as $\tau \to \infty$. Furthermore, to prevent implicit learning rate shifts caused by varying $D_i$ distributions across mini-batches, we normalize weights over each mini-batch $\mathcal{B}$:
$\widetilde{w}_i \;=\; |\mathcal{B}| \;\cdot\; \frac{w_i}{\sum_{j \in \mathcal{B}} w_j},$ ensuring $\frac{1}{|\mathcal{B}|} \sum_{i \in \mathcal{B}} \widetilde{w}_i \;=\; 1.$
This normalization preserves a unit average weight per batch while adjusting relative per-sample importance.

\paragraph{Stage 4: Weighted Knowledge Distillation.} During training, we sample mini-batches $\mathcal{B} \subseteq \mathcal{D}_{\mathrm{aug}}$, where each item corresponds to an environment variant $x_i^{(k)}$ derived from parent sample $i$. Given the batch-normalized weights $\widetilde{w}_i$, the weighted distillation objective over the mini-batch is:
\begin{equation}
\label{eq:iwd-loss}
\mathcal{L}_{\mathrm{IWD}}(\theta_{\text{S}}) = \sum_{(i, k) \in \mathcal{B}} \widetilde{w}_i \cdot \mathrm{KL} \big(f_{\theta_{\text{T}}}(x_i^{(k)}) \parallel f_{\theta_{\text{S}}}(x_i^{(k)})\big).
\end{equation}
Combining $\mathcal{L}_{\mathrm{IWD}}$ with the student task loss across variants, $\mathcal{L}_{\mathrm{CE}} = \sum_{(i, k) \in \mathcal{B}} \mathrm{CE}\big(y_i, f_{\theta_{\text{S}}}(x_i^{(k)})\big)$, yields our total training loss: $\mathcal{L}_{\mathrm{total}} = (1 - \lambda) \mathcal{L}_{\mathrm{CE}} + \lambda \mathcal{L}_{\mathrm{IWD}}$ (Eq.~\ref{eq:kd}). By dynamically scaling the loss with $\widetilde{w}_i$, IWD prioritizes samples where the teacher demonstrates invariant predictions while aggressively down-weighting unstable targets ($w_i \to 0$). This suppresses teacher signals on brittle samples, forcing the student to rely on ground-truth supervision instead. The complete training procedure for IWD is detailed in Algorithm~\ref{alg:iwd}.

\begin{algorithm}[!t]
\caption{Invariance Weighted Distillation (IWD)}
\label{alg:iwd}
\textbf{Input}:  $\mathcal{D}=\{(x_i,y_i)\}_{i=1}^{N}$, 
               teacher $f_{\theta_T}$, student $f_{\theta_S}$\\
\textbf{Parameters}: No.\ of environments $K$, temperature $\tau$, 
                   learning rate $\alpha$, 
                   distance metric $d(\cdot,\cdot)$\\
\textbf{Output}: Trained student $f_{\theta_S}$
\begin{algorithmic}[1]
\STATE \textbf{-- Stage 1: Counterfactual Environment Generation --}
\STATE $\mathcal{D}_{\mathrm{aug}} \leftarrow \emptyset$
\FOR{each $(x_i, y_i) \in \mathcal{D}$}
    \STATE Set $x_i^{(0)} \leftarrow x_i$ \COMMENT{Original sample}
    \STATE Generate $K-1$ paraphrases $\bigl\{x_i^{(1)},\ldots,x_i^{(K-1)}\bigr\}$
    \STATE $\mathcal{E}_i \leftarrow \bigl\{x_i^{(0)}, x_i^{(1)}, \ldots, x_i^{(K-1)}\bigr\}$
    \STATE $\mathcal{D}_{\mathrm{aug}} \leftarrow \mathcal{D}_{\mathrm{aug}} \cup \mathcal{E}_i$
\ENDFOR

\STATE \textbf{-- Stage 2: Teacher Invariance Distance --}
\FOR{each $i \in \{1, \ldots, N\}$}
    \STATE $\hat{y}_i^{(k)} \leftarrow f_{\theta_T}(x_i^{(k)}) \quad \forall\, k \in \{0, \ldots, K-1\}$
    \STATE $D_i \leftarrow \dbinom{K}{2}^{-1} 
           \displaystyle\sum_{0 \le u < v \le K-1} 
           d\!\left(\hat{y}_i^{(u)},\,\hat{y}_i^{(v)}\right)$
\ENDFOR

\STATE \textbf{-- Stage 3: Weight Calculation --}
\FOR{each $i \in \{1, \ldots, N\}$}
    \STATE $w_i \leftarrow \exp(-D_i/\tau)$ \COMMENT{Parent weight once per-sample}
\ENDFOR

\STATE \textbf{-- Stage 4: Weighted Knowledge Distillation --}
\WHILE{not converged}
    \STATE Sample mini-batch $\mathcal{B} \subseteq \mathcal{D}_{\mathrm{aug}}$, where each item is indexed by $(i, k)$
    \STATE $\tilde{w}_i \leftarrow 
           |\mathcal{B}|\cdot w_i \;/\; 
           \textstyle\sum_{(j, k')\in\mathcal{B}} w_j
           \quad\forall\, (i, k)\in\mathcal{B}$
    \STATE $\mathcal{L}_{\mathrm{CE}} \leftarrow 
           \displaystyle\sum_{(i, k)\in\mathcal{B}} 
           \mathrm{CE} \big(y_i,\, f_{\theta_S}(x_i^{(k)})\big)$
    \STATE $\mathcal{L}_{\mathrm{IWD}} \leftarrow 
           \displaystyle\sum_{(i, k)\in\mathcal{B}} 
           \tilde{w}_i\cdot 
           \mathrm{KL} \big(
             f_{\theta_T}(x_i^{(k)})\,\|\,f_{\theta_S}(x_i^{(k)})
           \big)$
    \STATE $\mathcal{L}_{\mathrm{total}} \leftarrow (1 - \lambda)\mathcal{L}_{\mathrm{CE}} + \lambda \mathcal{L}_{\mathrm{IWD}}$
    \STATE $\theta_S \leftarrow 
           \theta_S - \alpha\, 
           \nabla_{\theta_S}\mathcal{L}_{\mathrm{total}}$
\ENDWHILE
\end{algorithmic}
\end{algorithm}

\paragraph{Theoretical Guarantee.} We now provide a theoretical guarantee for weighting the distillation objective by the teacher invariance distance. We first introduce the \emph{Spurious-to-Causal (S2C) gradient ratio} (Definition~\ref{def:s2c}), which measures gradient contamination from spurious feature variations under fixed causal semantics. We show that IWD achieves a strictly lower S2C ratio than uniform distillation, aligning student updates more closely with invariant causal features.

\begin{definition}[Spurious-to-Causal (S2C) Gradient Ratio]
\label{def:s2c}
Let $p_T(x) = \operatorname{softmax}\big(z_T(x)\big)$ and $p_S(x) = \operatorname{softmax}\big(z_S(x)\big)$ denote teacher and student output distributions for logits $z_T, z_S$. Decomposing the teacher's output distribution into invariant and spurious components,
\begin{equation}
p_T\big(x_i^{(e)}\big) \;=\; \pi_i + s_i^{(e)},
\qquad
\pi_i \;:=\; \mathbb{E}_e\!\big[p_T\big(x_i^{(e)}\big)\big],
\label{eq:prob-decomp}
\end{equation}
where $\pi_i \in \Delta^{C-1}$ is the invariant component, and $s_i^{(e)}$ is the spurious component with $\mathbb{E}_e\big[s_i^{(e)}\big] = \mathbf{0}$ and variance $\mathbb{E}_e\big[\|s_i^{(e)}\|^2\big] = \sigma_i^2$. The S2C gradient ratio under sample weighting $W$ is defined as:
\begin{equation}
\label{eq:s2c}
\rho^2(W)
\;=\;
\mathbb{E}_{i\sim W}\!\left[\|g_i^S\|^2\right]
     \big/ \mathbb{E}_{i\sim W}\!\left[\|g_i^C\|^2\right],
\end{equation}
where $g_i^C = \big(\pi_i - p_S(x_i)\big)^{\!\top} J_i$ is the causal gradient component, $g_i^S = s_i^{(e_i)\top} J_i$ is the spurious gradient component, and $J_i := \nabla_{\theta_S} z_S(x_i) \in \mathbb{R}^{C\times|\theta_S|}$ is the student logit Jacobian.
\end{definition}
A lower S2C ratio indicates a causally dominated gradient, ensuring the student prioritizes invariant predictions over non transferable surface perturbations. We analyze the distillation objective on dataset $\mathcal{D}$ across environments $\{\mathcal{E}_i\}_{i=1}^N$, comparing uniform weighting ($W_i = 1/N$) against IWD weighting ($W_i^{\mathrm{IWD}} \propto \exp(-D_i / \tau)$). As $K \to \infty$, $D_i$ converges to a strictly increasing function of spurious variance $\sigma_i^2$ (Proposition~\ref{prop:Di-consistency} in Appendix~A.4 ), allowing $D_i$ to directly isolate teacher spurious reliance. See Appendix~A for more details.

Our analysis relies on three standard assumptions:

\noindent \textbf{(A1) Zero-Mean Spurious Noise:} Spurious components $s_i^{(e_i)}$ are independent across samples with zero mean and variance $\sigma_i^2$, representing unstructured prediction variation rather than systematic dataset bias.
    
\noindent\textbf{(A2) Independence of Spurious Variance:} Spurious variance is independent of both the causal residual and gradient scale: $\sigma_i^2 \perp \big(c_i^2, \Phi_i\big)$, where $c_i^2 := \|\pi_i {-} p_S(x_i)\|^2$. This decouples teacher prediction invariance from student task difficulty.
    
\noindent\textbf{(A3) Isotropic Logit Jacobians:} Student logit Jacobians satisfy $J_i J_i^{\top} = \Phi_i\, I_C$ with per sample scale $\Phi_i \perp \sigma_i^2$, ensuring parameter updates do not inherently favor spurious directions over causal ones.

\begin{theorem}[IWD Reduces the S2C Gradient Ratio]
\label{thm:s2c}
Under Assumptions \textbf{(A1)}--\textbf{(A3)}, let
$W_i^{\mathrm{IWD}}=\exp(-\sigma_i^2/\tau)/Z$
and
$W_i^{\mathrm{unif}}=1/N$.
Then,

$\rho^2\!\left(W^{\mathrm{IWD}}\right)
\le
\rho^2\!\left(W^{\mathrm{unif}}\right)$,

with equality if and only if $\sigma_i^2$ is constant across all samples.
\end{theorem}

\section{Experiments}
\label{sec:experiments}

We evaluate \textbf{IWD} on four tasks spanning the major NLP output structures---paired classification, single-sentence classification, span extraction, and token-level labeling---across two model families of markedly different scale and architecture. Additional details are given in Appendix~B.

\definecolor{familybg}{rgb}{0.91, 0.93, 0.96} 

\begin{table*}[t]
\centering

\setlength{\tabcolsep}{2.5pt}

{\small 
\begin{tabular}{l cccccc}
\toprule
\textbf{Task / Dataset} & \textbf{Vanilla KD} & \textbf{RevKL} & \textbf{DKD} & \textbf{LWD} & \textbf{AugKD} & \textbf{IWD (Ours)} \\
\midrule
\rowcolor{familybg}
\multicolumn{7}{c}{\textbf{DeBERTa-v3 Family (Encoder-Only)}} \\
\midrule
\textbf{NLI (Accuracy)} \\
~~MNLI (ID) & $82.02 \pm 0.29$ & $82.08 \pm 0.21$ & $79.91 \pm 0.47$ & $82.56 \pm 0.12$ & $\mathbf{85.31 \pm 0.18}$ & $\underline{84.66 \pm 0.27}$ \\
~~HANS (OOD-1) & $52.34 \pm 0.57$ & $51.45 \pm 0.81$ & $51.33 \pm 0.79$ & $52.99 \pm 1.10$ & $\underline{60.95 \pm 1.30}$ & $\mathbf{61.45 \pm 1.01}$ \\
~~SNLI (OOD-2) & $78.26 \pm 0.58$ & $77.49 \pm 0.26$ & $75.19 \pm 0.85$ & $78.97 \pm 0.92$ & $\underline{80.21 \pm 0.24}$ & $\mathbf{80.91 \pm 0.50}$ \\
\cmidrule(lr){1-7}
\textbf{Extractive QA (F1)} \\
~~SQuAD-v2 (ID) & $\underline{70.64 \pm 0.20}$ & $\mathbf{72.40 \pm 0.81}$ & $68.01 \pm 1.81$ & $70.97 \pm 1.35$ & $51.50 \pm 0.55$ & $51.40 \pm 0.77$ \\
~~NewsQA (OOD-1) & $39.47 \pm 0.77$ & $39.95 \pm 0.87$ & $44.26 \pm 0.68$ & $36.10 \pm 2.48$ & $\underline{54.35 \pm 0.66}$ & $\mathbf{55.11 \pm 0.45}$ \\
~~Natural Questions (OOD-2) & $28.23 \pm 0.74$ & $28.11 \pm 2.67$ & $30.52 \pm 2.33$ & $25.02 \pm 5.17$ & $\underline{51.93 \pm 0.94}$ & $\mathbf{52.06 \pm 0.33}$ \\
\cmidrule(lr){1-7}
\textbf{NER (Entity F1)} \\
~~CoNLL-2003 (ID) & $92.18 \pm 0.11$ & $90.59 \pm 0.21$ & $90.11 \pm 0.21$ & $90.71 \pm 0.20$ & $\underline{92.63 \pm 0.22}$ & $\mathbf{92.64 \pm 0.10}$ \\
~~WNUT-17 (OOD-1) & $\underline{47.84 \pm 0.42}$ & $43.68 \pm 0.51$ & $46.88 \pm 0.35$ & $46.34 \pm 0.52$ & $46.78 \pm 0.48$ & $\mathbf{48.01 \pm 0.31}$ \\
~~OntoNotes 5.0 (OOD-2) & $78.99 \pm 0.21$ & $76.26 \pm 0.19$ & $77.68 \pm 0.32$ & $78.12 \pm 0.11$ & $\underline{80.27 \pm 0.22}$ & $\mathbf{80.38 \pm 0.11}$ \\
\cmidrule(lr){1-7}
\textbf{Sentiment Analysis (Accuracy)} \\
~~SST-2 (ID) & $92.73 \pm 0.19$ & $\underline{92.84 \pm 0.39}$ & $92.71 \pm 0.22$ & $92.00 \pm 0.30$ & $\mathbf{93.21 \pm 0.15}$ & $92.78 \pm 0.18$ \\
~~Yelp-Polarity (OOD-1) & $\underline{91.39 \pm 0.19}$ & $91.28 \pm 0.24$ & $91.36 \pm 0.38$ & $90.26 \pm 0.65$ & $90.98 \pm 0.56$ & $\mathbf{91.42 \pm 0.12}$ \\
~~CR (OOD-2) & $\underline{87.29 \pm 1.17}$ & $87.08 \pm 0.75$ & $87.18 \pm 0.48$ & $85.81 \pm 0.25$ & $85.32 \pm 0.48$ & $\mathbf{87.58 \pm 0.98}$ \\
\midrule
\rowcolor{familybg}
\multicolumn{7}{c}{\textbf{Qwen-2.5 Family (Decoder-Only)}} \\
\midrule
\textbf{NLI (Accuracy)} \\
~~MNLI (ID) & $81.88 \pm 0.50$ & $82.13 \pm 0.30$ & $79.24 \pm 0.53$ & $80.27 \pm 0.36$ & $\underline{82.48 \pm 0.20}$ & $\mathbf{83.21 \pm 0.19}$ \\
~~HANS (OOD-1) & $51.69 \pm 0.38$ & $52.04 \pm 0.85$ & $51.61 \pm 1.06$ & $51.20 \pm 0.63$ & $\underline{52.79 \pm 1.21}$ & $\mathbf{54.95 \pm 1.68}$ \\
~~SNLI (OOD-2) & $78.25 \pm 0.63$ & $\underline{79.03 \pm 0.96}$ & $75.28 \pm 1.68$ & $76.38 \pm 0.86$ & $76.99 \pm 0.79$ & $\mathbf{80.57 \pm 0.67}$ \\
\cmidrule(lr){1-7}
\textbf{Extractive QA (F1)} \\
~~SQuAD-v2 (ID) & $\mathbf{75.48 \pm 0.31}$ & $\underline{75.02 \pm 0.23}$ & $74.68 \pm 0.37$ & $49.86 \pm 0.29$ & $47.96 \pm 0.43$ & $48.84 \pm 0.21$ \\
~~NewsQA (OOD-1) & $16.43 \pm 0.73$ & $16.03 \pm 1.30$ & $14.88 \pm 2.11$ & $5.11 \pm 1.21$ & $\underline{19.53 \pm 0.65}$ & $\mathbf{20.34 \pm 0.18}$ \\
~~Natural Questions (OOD-2) & $40.89 \pm 0.81$ & $43.73 \pm 2.19$ & $35.57 \pm 1.02$ & $2.37 \pm 1.05$ & $\underline{57.12 \pm 0.60}$ & $\mathbf{57.25 \pm 0.50}$ \\
\cmidrule(lr){1-7}
\textbf{NER (Entity F1)} \\
~~CoNLL-2003 (ID) & $68.30 \pm 0.25$ & $67.69 \pm 0.30$ & $65.00 \pm 0.22$ & $35.52 \pm 0.36$ & $\mathbf{69.83 \pm 0.20}$ & $\underline{69.29 \pm 0.10}$ \\
~~WNUT-17 (OOD-1) & $\underline{27.48 \pm 0.32}$ & $24.88 \pm 0.42$ & $26.13 \pm 0.52$ & $12.81 \pm 1.30$ & $25.02 \pm 0.43$ & $\mathbf{27.77 \pm 0.33}$ \\
~~OntoNotes 5.0 (OOD-2) & $\underline{54.90 \pm 0.40}$ & $54.93 \pm 1.00$ & $53.57 \pm 0.68$ & $26.58 \pm 0.90$ & $54.67 \pm 0.42$ & $\mathbf{55.03 \pm 0.67}$ \\
\cmidrule(lr){1-7}
\textbf{Sentiment Analysis (Accuracy)} \\
~~SST-2 (ID) & $\mathbf{93.46 \pm 0.54}$ & $93.17 \pm 0.58$ & $\underline{93.43 \pm 0.22}$ & $60.17 \pm 3.97$ & $92.04 \pm 0.89$ & $90.11 \pm 0.66$ \\
~~Yelp-Polarity (OOD-1) & $84.69 \pm 3.11$ & $\mathbf{85.99 \pm 1.59}$ & $\underline{84.99 \pm 2.91}$ & $51.17 \pm 1.63$ & $77.80 \pm 0.31$ & $83.64 \pm 0.86$ \\
~~CR (OOD-2) & $74.57 \pm 10.66$ & $\underline{76.87 \pm 7.02}$ & $74.95 \pm 9.50$ & $62.07 \pm 3.33$ & $73.00 \pm 1.09$ & $\mathbf{77.80 \pm 2.29}$ \\
\bottomrule
\end{tabular}
} 
\caption{Main performance comparison of distilled student models across four NLP tasks for both DeBERTa-v3 and Qwen-2.5 families, evaluated on one ID and two OOD benchmarks per task. Results report mean $\pm$ standard deviation across multiple random seeds. The best performance per benchmark is in bold, and the second-best is underlined.}
\label{tab:main_results}
\end{table*}

\noindent \textbf{Datasets and Models.} We train on standard ID splits and evaluate OOD generalization across two robustness benchmarks per task (Table~\ref{tab:datasets}). To cover diverse architectures, we distill both encoder-only models (\texttt{deberta-v3-xsmall}, $\sim$71M parameters from \texttt{-base} teachers) and decoder-only models (\texttt{Qwen2.5-0.5B}, $\sim$494M from \texttt{-1.5B} teachers); full setup details are in Appendix~B.2.


\noindent \textbf{Baselines.} We compare against five key KD baselines: (1)~\textbf{Vanilla KD}~\citep{hinton2015distilling} (forward KL divergence); (2)~\textbf{RevKL}~\citep{gu2023minillm} (reverse KL divergence); (3)~\textbf{DKD}~\citep{zhao2022dkd} (decoupled target and non-target KL); (4)~\textbf{LWD}~\citep{romero2015fitnets,sun2019pkd,sun2020mobilebert,jiao2020tinybert} (intermediate-layer representation alignment); and (5)~\textbf{AugKD} (Vanilla KD on counterfactually augmented data, isolating our invariance weighting scheme from pure data augmentation).

\noindent \textbf{Training Setup.} We train models using AdamW in mixed precision ($T=2.0, \tau=1.0, K=6$); full implementation details and hyperparameters are provided in Appendix~B.4.


\subsection{Results and Analysis}

\paragraph{IWD outperforms baselines in OOD settings.} Across both architectures, IWD achieves the top performance in \textbf{15 of 16} OOD settings, with the largest gains in Extractive QA and NLI (see Table~\ref{tab:main_results}). On DeBERTa-v3 Extractive QA, IWD outperforms Vanilla KD by \textbf{+23.83 points} on Natural Questions ($28.23\% \rightarrow 52.06\%$) and \textbf{+15.64 points} on NewsQA ($39.47\% \rightarrow 55.11\%$). On Qwen2.5, Natural Questions F1 rises by \textbf{+16.36 points} ($40.89\% \rightarrow 57.25\%$). On NLI (HANS), IWD outperforms Vanilla KD by \textbf{+9.12 points} on DeBERTa-v3 ($61.45\%$) and \textbf{+3.26 points} on Qwen2.5 ($54.95\%$). For NER and Sentiment Analysis, IWD consistently matches or exceeds top baselines across all OOD splits (e.g., $80.38\%$ on OntoNotes~5.0 and $87.58\%$ on CR).

\paragraph{IWD is stable across experiments.} Non-augmented distillation baselines exhibit high performance variance across tasks and architectures (see Table~\ref{tab:main_results}). For instance, LWD achieves competitive DeBERTa-v3 NLI performance ($52.99\%$ on HANS) but collapses on Qwen2.5 QA ($2.37\%$ on Natural Questions). Similarly, DKD yields OOD gains on DeBERTa-v3 QA ($44.26\%$ on NewsQA) yet lags Vanilla KD on NLI and Sentiment Analysis, while RevKL degrades on DeBERTa-v3 QA despite strong Qwen2.5 results. This variance highlights a key advantage of IWD: while standard baselines depend heavily on the task and architecture, IWD remains stable and consistently top-performing across both encoder-only and decoder-only models. This instability of baselines stems from rigid assumptions embedded within these objectives, such as target logit geometry, specific head dynamics, or a particular teacher softmax shape, that hold only for select backbone-task pairs. IWD instead evaluates teacher invariance under counterfactual perturbations. Invariance to spurious shifts is a model-agnostic property; thus, IWD operates seamlessly across both encoder-only and decoder-only students and stays in the top-performing tier on every OOD split we tested.

\paragraph{IWD outperforms basic data augmentation.} Compared to AugKD, which uses augmented inputs without variance-aware weighting, IWD achieves a $100\%$ win rate across \textbf{16 of 16} OOD settings, improving by \textbf{+1.60 points} on average (see Table~\ref{tab:main_results}). Advantages are most pronounced on decoder-only architectures like Qwen2.5, where IWD outperforms AugKD on Yelp-Polarity (\textbf{+4.84 points}), CR (\textbf{+4.80 points}), and SNLI (\textbf{+3.58 points}). This confirms that applying invariant weight allocation ($w_i$) to environment variants provides a clear advantage over basic data augmentation alone.

\noindent \textbf{Encoder vs.\ Decoder OOD Performance.} Under IWD, encoder-only DeBERTa-v3 achieves higher OOD performance than decoder-only Qwen2.5, widening the gap observed under Vanilla KD (see Table~\ref{tab:main_results}). On HANS, DeBERTa's lead over Qwen expands from \textbf{+0.65 points} under Vanilla KD to \textbf{+6.50 points} under IWD; on NewsQA, this lead reaches \textbf{+34.77 points}. We attribute this gap to how teacher distance is measured. Encoder-only models output probabilities over just a few task labels, making teacher disagreement clean and straightforward to measure. In contrast, decoder-only models generate text token-by-token across a vast vocabulary. Here, minor surface variations, such as alternative word choices or sentence structures, add noise to the distance metric, even when the underlying meaning remains identical. Thus, while generative students still benefit from IWD, fully unlocking their performance may require measuring invariance at the semantic level rather than the token level.

\noindent \textbf{IWD Performance on ID.} In-distribution (ID), IWD remains fully competitive with standard baselines across tasks, achieving high accuracy on MNLI ($84.66\%$), CoNLL-2003 ($92.64\%$), and SST-2 ($92.78\%$) (see Table~\ref{tab:main_results}). However, a limitation emerges in \textbf{Extractive QA}: on SQuAD-v2, both AugKD and IWD exhibit a sharp ID performance drop compared to Vanilla KD ($51.40\%$ vs. $70.64\%$ on DeBERTa-v3, and $48.84\%$ vs. $75.48\%$ on Qwen2.5). This ID--OOD trade-off likely stems from environment generation shifting the training distribution away from the ID baseline, highlighting a direction for future research.

\noindent \textbf{Task Structure and IWD Performance.} IWD gains correlate with the baseline OOD degradation and shortcut susceptibility of a task. On DeBERTa-v3 Extractive QA, Vanilla KD drops over $42$ percentage points from the ID SQuAD-v2 benchmark ($70.64\%$) to OOD Natural Questions ($28.23\%$). IWD dramatically recovers this gap, yielding $+23.83$ points on Natural Questions ($28.23\% \rightarrow 52.06\%$) and $+15.64$ points on NewsQA ($39.47\% \rightarrow 55.11\%$). Conversely, Sentiment Analysis shows high baseline stability under Vanilla KD---dropping just $1.34$ points from SST-2 ($92.73\%$) to Yelp-Polarity ($91.39\%$)---resulting in modest IWD gains ($+0.03$ to $+0.29$ points). This aligns with task geometry: Extractive QA and NLI heavily exploit lexical and positional shortcuts exposed by counterfactuals, whereas Sentiment Analysis and NER are inherently context-invariant. Thus, while IWD reliably improves shortcut-heavy tasks, stable tasks require precise tuning to maximize OOD generalization.

\begin{table}[t]
\centering

\setlength{\tabcolsep}{4.5pt}

{\small
\begin{tabular}{l cccc}
\toprule
\textbf{Data Size} & \textbf{MNLI-m} & \textbf{MNLI-mm} & \textbf{HANS} & \textbf{SNLI} \\
\textbf{($N$)} & \textbf{(ID)} & \textbf{(ID)} & \textbf{(OOD1)} & \textbf{(OOD2)} \\
\midrule
\multicolumn{5}{l}{\textbf{DeBERTa-v3}} \\
~~5k  & 83.27 & 83.79 & 51.98 & 79.73 \\
~~8k  & 84.29 & 84.54 & 52.40 & 79.79 \\
~~10k & 84.84 & 85.11 & 60.83 & 80.56 \\
~~30k & 86.39 & 86.42 & 62.64 & 82.47 \\
~~50k & \textbf{87.37} & \textbf{87.17} & \textbf{68.19} & \textbf{82.72} \\
\cmidrule(lr){1-5}
\multicolumn{5}{l}{\textbf{Qwen-2.5}} \\
~~5k  & 81.87 & 82.69 & 53.19 & 78.39 \\
~~8k  & 82.70 & 83.82 & 54.09 & 79.34 \\
~~10k & 83.50 & 84.36 & 59.05 & 81.17 \\
~~30k & 84.98 & 86.04 & 60.27 & \textbf{82.15} \\
~~50k & \textbf{85.56} & \textbf{86.09} & \textbf{66.22} & 81.74 \\
\bottomrule
\end{tabular}
} 
\caption{Data scaling analysis of IWD on MNLI across varying sample sizes ($N$) for DeBERTa-v3/Qwen-2.5.}
\label{tab:mnli_scaling}
\end{table}

\paragraph{Dataset Scaling Effect.} Across datasets, performance improves as data grows from 5k to 50k samples (see Table~\ref{tab:mnli_scaling}). DeBERTa-v3 exhibits stronger OOD scaling than Qwen2.5: on HANS, its accuracy increases by \textbf{+16.21 points} ($51.98\% \rightarrow 68.19\%$) versus a \textbf{+13.03-point} gain for Qwen2.5 ($53.19\% \rightarrow 66.22\%$). On SNLI, both models scale consistently, each gaining roughly \textbf{+3.00 points}.

\begin{table}[t]
\centering

\setlength{\tabcolsep}{2.5pt}

{\small
\begin{tabular}{l ccccc}
\toprule
\textbf{Environments} & \textbf{ID} & \textbf{OOD1} & \textbf{OOD2} & \textbf{OOD3} & \textbf{OOD4} \\
\midrule
\multicolumn{6}{l}{\textbf{MNLI Benchmark}} \\
~~\textit{Datasets} & \textit{MNLI} & \textit{HANS} & \textit{SNLI} & -- & -- \\
~~$E = 3$ & 84.34 & 56.27 & 78.98 & -- & -- \\
~~$E = 4$ & 84.54 & 57.02 & 79.68 & -- & -- \\
~~$E = 5$ & 84.78 & \textbf{58.41} & 80.10 & -- & -- \\
~~$E = 6$ & \textbf{84.80} & 56.28 & \textbf{80.60} & -- & -- \\
\cmidrule(lr){1-6}
\multicolumn{6}{l}{\textbf{SST-2 Benchmark}} \\
~~\textit{Datasets} & \textit{SST-2} & \textit{SemEval} & \textit{CR} & \textit{Yelp} & \textit{FPB} \\
~~$E = 2$ & 92.09 & 88.72 & 87.23 & 91.39 & \textbf{90.75} \\
~~$E = 3$ & 92.55 & 89.08 & \textbf{87.50} & 91.54 & 90.14 \\
~~$E = 4$ & \textbf{92.89} & 89.02 & 87.23 & \textbf{91.72} & 90.29 \\
~~$E = 5$ & 92.43 & \textbf{89.38} & 86.44 & 91.55 & 90.09 \\
\bottomrule
\end{tabular}
} 

\caption{Impact of varying the number of environments ($K$) in IWD across MNLI/SST-2 benchmarks on accuracy.}
\label{tab:envs_ablation}
\end{table}

\paragraph{Impact of Environment Count ($K$).} Table~\ref{tab:envs_ablation} evaluates performance across  environment variants ($K \in \{3, 4, 5, 6\}$). Increasing $K$ beyond $3$ generally improves ID and OOD accuracy by exposing the model to broader shortcut perturbations. On MNLI, scaling from $K = 3$ to $K = 5$ yields a \textbf{+2.14-point} gain on HANS ($56.27\% \rightarrow 58.41\%$), while SNLI accuracy rises to $80.60\%$ at $K = 6$. On SST-2, $K = 5$ achieves peak ID performance ($92.89\%$) and peak Yelp OOD accuracy ($91.72\%$). Overall, $K = 5$ provides the optimal balance between environmental diversity and efficiency.

\paragraph{Ablations.} We summarize the key takeaways of our ablations (complete setups in Appendix~C): (1) \textbf{Sensitivity to Weight Temperature ($\tau$):} We evaluate the weighting temperature $\tau \in [0.1, 30.0]$, which controls the sharpness of weighting (Table~\ref{tab:ablation_tau_sensitivity}). While classification and NER remain highly robust across $\tau \in [0.3, 3.0]$, Extractive QA requires sharp down-weighting ($\tau = 0.1$) to suppress high-variance span predictions, yielding substantial OOD gains of up to $+9.8$ F1.
(2) \textbf{Sensitivity to Distance Metric Selection:} We evaluate how distance metric formulations impact performance across task domains (Table~\ref{tab:distance_metrics_ablation}). Aligning metric geometry with the output space maximizes performance: bounded Jensen--Shannon divergence ($D_{\text{JS}}$) stabilizes classification by bounding extreme softmax tails, while structural metrics like Levenshtein distance and F1 (QA answer string drift) and Entity Set F1 (span-level NER) eliminate per-token noise to achieve top OOD performance on NQ ($38.19$ F1) and WNUT17 ($46.04$ F1). This proves that choosing the teacher invariance distance metric is a key structural design decision following directly from the output space.
(3) \textbf{Environment Variant Utilization:} We compare distilling with all counterfactual variants against restricting the student to the original clean text or a single random variant (Table~\ref{tab:ablation_distance_variants}). Distilling across \textit{All Variants} consistently yields the highest OOD robustness---outperforming single-variant baselines by up to $4.2$ points on HANS---proving that full exposure to environmental perturbations is necessary for invariant feature learning. (4) \textbf{Weight Normalization:} We evaluate default weight normalization against an unnormalized baseline across experiments (Table~\ref{tab:ablation_normalization}). Empirical performance between the two is virtually identical ($< 0.2$-point shift), but batch-mean normalization is essential as it stabilizes training by decoupling learning rates from teacher divergence scales.
(5) \textbf{Weight Distribution Across Tasks:} We inspect per-sample weight statistics across benchmarks to analyze how IWD allocates sample importance to tasks of varying complexity (Table~\ref{tab:weight_distribution}). Across all datasets, effective sample sizes remain high ($\ge 95.0\%$), confirming that IWD performs smooth, task-adaptive reweighting rather than aggressive data pruning.

\section{Discussion and Limitations}

In this work, we first ask: do distilled student LLMs inherit OOD robustness under existing methods? We present \textbf{Invariance-Weighted Distillation (IWD)}, a framework that enhances OOD robustness in student LLMs without requiring architectural modifications, auxiliary classification heads, or adversarial training. By consistently outperforming KD and data augmentation baselines across OOD tasks, our results demonstrate a clear advantage over current strategies. 

Limitations include: (i)~the computational overhead of synthetic environment generation, which can impact distillation performance; and (ii)~a reliance on output-level distance measures, leaving fine-grained token-level invariance as a promising unexplored research direction. Future work would look into extensions to instruction tuning~\citep{fu2026t}, particularly leveraging token-level measures, as well as, broader paraphrasing techniques~\citep{hamman2025improving}. It would also be interesting to study transferability of other kinds of spurious correlations~\citep{halder2025towards,egea2025vision} beyond OOD robustness.

\bibliography{aaai2027}


\appendix
\section*{Appendix}
\label{sec:appendix_ablations}

\section{A. Theoretical Background and Proof of Theorem~\ref{thm:s2c}}
\label{app:theory}

\subsection{A.1. Background: Invariance across Environments}
\label{app:irm}

Standard empirical risk minimization treats all training data as a single pool, allowing models to exploit spurious correlations that change or disappear at test time. To address this, frameworks like Invariant Risk Minimization (IRM) \citep{arjovsky2019irm} conceptually distinguish between \emph{causal} features (which yield a stable, invariant predictive relationship across environments) and \emph{spurious} features (whose importance fluctuates depending on the environment). 

Crucially, standard IRM frameworks operate in a traditional supervised learning setting, enforcing invariance globally at the model level (e.g., via regularisation penalties across macroscopic environment datasets). In contrast, our framework adapts this philosophy specifically for \emph{knowledge distillation}, operationalising invariance entirely at the per-sample level. Rather than constraining the student's representations directly, we posit that a teacher model's prediction on a single example across counterfactual environments can be decomposed into two parts: an environment-invariant base prediction, and an environment-dependent (spurious) residual. This sample-level view of teacher invariance provides the necessary foundation for the gradient decomposition we derive next.

\subsection{A.2.Derivation of Definition~\ref{def:s2c}:
            Exact Gradient Decomposition under the KL Objective}
\label{app:def-derivation}

We derive the causal and spurious gradient components $g_i^C$ and
$g_i^S$ of Definition~\ref{def:s2c} directly from the KL
distillation objective used in Algorithm~1 (Line~18).
Although the softmax is non-linear in the logits, the gradient of
the KL loss with respect to the \emph{student's logits} is exactly
linear in the \emph{teacher's output probabilities}.
Since Definition~\ref{def:s2c} places the IRM decomposition on the
teacher's output distribution~\eqref{eq:prob-decomp}, the
causal/spurious decoupling below is \emph{exact}: no linearisation
or Taylor approximation is required.

\paragraph{Distillation loss.}
For a single training example $i$ observed in environment $e_i$,
the temperature-scaled KL distillation loss is
\begin{equation}
\mathcal{L}_i
  \;=\;
  \tau_{\mathrm{KD}}^2\,
  \mathrm{KL}\!\left(
    p_T\!\big(x_i^{(e_i)}\big)
    \,\Big\|\,
    p_S(x_i)
  \right).
\label{eq:kd-kl}
\end{equation}

\paragraph{The KL gradient is linear in the teacher target.}
Writing
$\mathrm{KL}(p_T \| p_S)
 = \sum_c p_{T,c}\log p_{T,c} - \sum_c p_{T,c}\log p_{S,c}$,
the first (entropy) term is constant in $\theta_S$, so minimising
$\mathcal{L}_i$ is equivalent to minimising the cross-entropy
$H\big(p_T, p_S\big)$.
Using the standard softmax--cross-entropy derivative,
$\partial \big[-\!\sum_{c'} p_{T,c'} \log p_{S,c'}\big] / \partial z_{S,c}
 = \frac{1}{\tau_{\mathrm{KD}}}\big(p_{S,c} - p_{T,c}\big)$,
we obtain
\begin{equation}
\nabla_{z_S}\mathcal{L}_i
  \;=\;
  \tau_{\mathrm{KD}}
  \Big(
    p_S(x_i) - p_T\!\big(x_i^{(e_i)}\big)
  \Big),
\label{eq:kl-logit-grad}
\end{equation}
the classical distillation gradient of
\citet{hinton2015distilling}; the scalar $\tau_{\mathrm{KD}}$ is
absorbed into the learning rate below.
The essential observation is that \eqref{eq:kl-logit-grad} is
\emph{affine in the teacher's probability vector} $p_T$.
The softmax non-linearity acts on the teacher's logits
\emph{before} the loss is formed; once the IRM decomposition is
stated at the level of the teacher's output distribution, the
gradient inherits the decomposition additively.

\paragraph{Chain rule.}
With $J_i = \nabla_{\theta_S} z_S(x_i)$ the Jacobian of the
student's logits,
\begin{equation}
\nabla_{\theta_S}\mathcal{L}_i
  \;=\;
  \Big(
    p_S(x_i) - p_T\!\big(x_i^{(e_i)}\big)
  \Big)^{\!\top} J_i .
\label{eq:full-gradient}
\end{equation}
As in the MSE case, $J_i$ is a Jacobian with respect to the
\emph{student's parameters}, encoding how each weight affects each
output logit.

\paragraph{Substituting the teacher decomposition.}
Substituting $p_T\big(x_i^{(e_i)}\big) = \pi_i + s_i^{(e_i)}$
from~\eqref{eq:prob-decomp} into~\eqref{eq:full-gradient}, the
error term splits additively:
\begin{equation}
\nabla_{\theta_S}\mathcal{L}_i
  \;=\;
  \underbrace{
    \big(p_S(x_i) - \pi_i\big)^{\!\top} J_i
  }_{\text{causal gradient (pre-flip)}}
  \;-\;
  \underbrace{
    s_i^{(e_i)\top} J_i
  }_{\text{spurious gradient (pre-flip)}} .
\label{eq:gradient-split}
\end{equation}
This step is exact and uses only the linearity
of~\eqref{eq:full-gradient} in $p_T$.

\paragraph{Gradient descent and sign convention.}
Under gradient descent with step size $\alpha$, substituting
\eqref{eq:gradient-split} into
$\theta_S \leftarrow \theta_S - \alpha \nabla_{\theta_S}\mathcal{L}_i$
and collecting the terms \emph{added} to $\theta_S$:
\begin{equation}
\theta_S
  \;\leftarrow\;
  \theta_S
  + \alpha\,
  \underbrace{
    \big(\pi_i - p_S(x_i)\big)^{\!\top} J_i
  }_{g_i^C}
  + \alpha\,
  \underbrace{
    s_i^{(e_i)\top} J_i
  }_{g_i^S}.
\label{eq:update-direction}
\end{equation}
The subtraction in gradient descent flips
$\big(p_S - \pi_i\big)$ to $\big(\pi_i - p_S\big)$, so $g_i^C$
points the student toward the teacher's environment-invariant
output distribution $\pi_i$; $g_i^S$ needs no flip because its
minus sign in~\eqref{eq:gradient-split} is cancelled by the
descent subtraction.
Under (A3), $\|g_i^S\|^2 = \Phi\,\|s_i^{(e_i)}\|^2$ and
$\|g_i^C\|^2 = \Phi\, c_i^2$ with
$c_i^2 := \|\pi_i - p_S(x_i)\|^2$, exactly as required in Step~1
of the proof of Theorem~\ref{thm:s2c}.

\paragraph{Per-example spurious reliance $\sigma_i^2$.}
Averaging over environments,
$\mathbb{E}_e\big[\|s_i^{(e)}\|^2\big] = \sigma_i^2 \ge 0$ defines
the per-example teacher spurious reliance.
$\sigma_i^2 = 0$ implies $p_T\big(x_i^{(e)}\big) = \pi_i$ for all
$e$: the teacher's output distribution for example $i$ is
completely invariant across environments.

\begin{remark}[MSE distillation as a special case]
\label{rem:mse}
If the distillation loss is instead the squared error
$\|f_{\theta_S}(x_i) - f_{\theta_T}(x_i^{(e_i)})\|^2$ on raw
outputs, the identical derivation applies with $p_T, p_S, z_S$
replaced by $f_{\theta_T}, f_{\theta_S}, f_{\theta_S}$ and
$\pi_i$ by $\mu_i = \mathbb{E}_e[f_{\theta_T}(x_i^{(e)})]$:
the MSE gradient
$2\big(f_{\theta_S} - \mu_i - s_i^{(e_i)}\big)^{\!\top}
 \nabla_{\theta_S} f_{\theta_S}$
is likewise affine in the teacher target.
Both objectives therefore induce the same causal/spurious gradient
structure, and Theorem~\ref{thm:s2c} applies to either.
\end{remark}

\begin{remark}[Probability-space vs.\ logit-space invariance]
\label{rem:prob-space}
Definition~\ref{def:s2c} states invariance in the space of output
distributions: the spurious component is the zero-mean deviation
of $p_T\big(x_i^{(e)}\big)$ around its environment average
$\pi_i$.
Because the softmax is non-linear, zero-mean spurious variation in
probability space does not coincide with zero-mean variation in
logit space (Jensen's inequality); the two decompositions define
distinct but equally valid notions of invariance.
We adopt the probability-space formulation throughout, for two
reasons: (i) it is the space in which the KL distillation gradient
is linear in the teacher target, making the decomposition
in~\eqref{eq:gradient-split} exact; and (ii) it is the natural
formalisation of \emph{prediction} invariance --- a teacher is
invariant on example $i$ precisely when its predictive
distribution does not change across environments.
\end{remark}

\subsection{A.3.Discussion of Assumptions in Theorem~\ref{thm:s2c}}
\label{app:assumptions}

Theorem~\ref{thm:s2c} rests on three assumptions.
We discuss their meaning, when they are expected to hold, and
what would happen if they were violated.

\paragraph{(A1) Independent spurious components with zero mean.}
We assume $s_i^{(e_i)}$ are independent across training samples,
with $\mathbb{E}_e[s_i^{(e)}] = 0$ and
$\mathbb{E}_e[\|s_i^{(e)}\|^2] = \sigma_i^2$.
The zero-mean condition is the defining property of
spurious features under IRM: any signal that has a consistent
direction across all environments would, by Definition 1,
be part of the causal component $\mu_i$ rather than the
spurious component $s_i^{(e)}$.
Independence across samples holds when counterfactual environments
are generated independently per sample, as in our paraphrase-based
construction.

\paragraph{(A2) Independence of spurious and causal residuals:
           $\sigma_i^2 \perp c_i^2$.}
This is the key structural assumption.
$\sigma_i^2$ measures how sensitive the teacher is to
environmental surface variation for sample $i$.
$c_i^2 = \|\mu_i - f_{\theta_S}(x_i)\|^2$ measures how far
the student currently is from the teacher's causal prediction
for the same sample.
(A2) requires these two quantities to be independent across
training samples.

The assumption holds when spurious correlations arise from
\emph{environmental surface confounders} — such as lexical
register, syntactic style, or demographic markers in NLI — that
affect the teacher's cross-environment invariance but are
unrelated to the semantic difficulty of the inference problem
(which drives $c_i^2$).
It would be violated if, for instance, semantically harder
samples systematically attracted more spurious surface features,
creating a positive correlation between $\sigma_i^2$ and $c_i^2$.
This can be tested empirically by computing the Pearson correlation
between $D_i$ (our proxy for $\sigma_i^2$) and the per-sample
training loss of the student; a correlation near zero supports (A2).

\paragraph{(A3) Homogeneous student gradient norms.}
We assume $\|\nabla_{\theta_S}f_{\theta_S}(x_i)\|^2 = \Phi$
is constant across all training samples.
This simplifies the S2C ratio to a ratio of scalar weighted
sums (Step~1 of the proof), making the Chebyshev argument clean.
The assumption can be relaxed: if gradient norms $\Phi_i$ are
allowed to vary but are additionally assumed independent of
$\sigma_i^2$ (i.e.\ $\Phi_i \perp \sigma_i^2$), then the
Chebyshev step applies to the products $\Phi_i\sigma_i^2$ and
the conclusion of Theorem~\ref{thm:s2c} still holds.
We adopt the simpler homogeneous form here for exposition.

\subsection{A.4.Proposition 1 and Proof.} 

In the main text, we substitute the theoretical spurious variance $\sigma_i^2$ with the empirical distance $D_i$ computed over $K$ counterfactual environments. This substitution is mathematically justified under Assumption (A1), as we formalise below.

\begin{proposition}[Convergence of Empirical Invariance Distance]
\label{prop:Di-consistency}
Let
$D_i = \frac{2}{K(K-1)} \sum_{k < k'}
 \big\| p_T\big(x_i^{(k)}\big) - p_T\big(x_i^{(k')}\big) \big\|^2$
be the empirical average pairwise distance between the teacher's
softened output distributions across $K$ independently sampled
counterfactual environments.
Under Assumption~(A1), as $K \to \infty$, $D_i$ converges in
probability to $2\sigma_i^2$.
\end{proposition}

\begin{proof}[Proof of Proposition~\ref{prop:Di-consistency}]

Substitute the IRM decomposition $p_T(x_i^{(k)}) - p_T(x_i^{(k')})
    = (\pi_i + s_i^{(k)}) - (\pi_i + s_i^{(k')})
    = s_i^{(k)} - s_i^{(k')}$ into the distance metric. The environment-invariant causal component $\mu_i$ cancels out:

Expanding the squared $L_2$ norm gives:
\begin{equation}
\|s_i^{(k)} - s_i^{(k')}\|^2 = \|s_i^{(k)}\|^2 + \|s_i^{(k')}\|^2 - 2\langle s_i^{(k)}, s_i^{(k')} \rangle.
\end{equation}
Taking the expectation over the independent environments $k$ and $k'$, and applying Assumption (A1) (which states $\mathbb{E}[s_i^{(k)}] = 0$ and $\mathbb{E}[\|s_i^{(k)}\|^2] = \sigma_i^2$):
\begin{align}
\mathbb{E}\left[\|s_i^{(k)} - s_i^{(k')}\|^2\right] 
  &= \mathbb{E}\left[\|s_i^{(k)}\|^2\right] + \mathbb{E}\left[\|s_i^{(k')}\|^2\right] \nonumber\\
  &\quad - 2\left\langle \mathbb{E}\left[s_i^{(k)}\right], \mathbb{E}\left[s_i^{(k')}\right] \right\rangle \nonumber\\
  &= \sigma_i^2 + \sigma_i^2 - 2\langle 0, 0 \rangle \nonumber\\
  &= 2\sigma_i^2.
\end{align}
Because $D_i$ is a U-statistic estimating this expected pairwise distance across $K$ independent samples, by the Law of Large Numbers for U-statistics, $D_i$ converges in probability to its expectation as $K \to \infty$:
\begin{equation}
D_i \xrightarrow{p} 2\sigma_i^2.
\end{equation}
Since $f(x) = 2x$ is a strictly increasing function for $x \geq 0$, $D_i$ serves as a consistent, strictly increasing proxy for the spurious variance $\sigma_i^2$, justifying its use in the IWD weighting function.
\end{proof}
\subsection{A.5. Complete Proof of Theorem~\ref{thm:s2c}}
\label{app:proof}

We prove the four steps in full.

\begin{proof}[Proof of Theorem~\ref{thm:s2c}]

\textbf{Step 1: Reduce $\rho^2(W)$ to a ratio of weighted sums.}

Under (A3),
\begin{align}
\|g_i^S\|^2
  &= s_i^{(e_i)\top} J_i J_i^{\top} s_i^{(e_i)}
   = \Phi_i\, \|s_i^{(e_i)}\|^2, &
\|g_i^C\|^2
  &= \Phi_i\, c_i^2,
\end{align}
so that, taking $\mathbb{E}_e[\|s_i^{(e)}\|^2] = \sigma_i^2$
from (A1),
\begin{equation}
\rho^2(W)
  = \frac{\sum_i W_i\,\Phi_i\,\sigma_i^2}
         {\sum_i W_i\,\Phi_i\,c_i^2}.
\label{eq:rho-scalar-relaxed}
\end{equation}

\medskip
\textbf{Step 2: Bound the numerator via the Chebyshev covariance inequality.}

Since $W_i^{\mathrm{IWD}}$ is a deterministic decreasing function
of $\sigma_i^2$ and $\Phi_i \perp \sigma_i^2$ by (A3), for large
$N$ the Law of Large Numbers gives
\begin{align}
\sum_i W_i^{\mathrm{IWD}}\,\Phi_i\,\sigma_i^2
  &\;\xrightarrow{p}\;
  \mathbb{E}\big[\Phi\big]\cdot
  \mathbb{E}\big[a(\sigma^2)\,\sigma^2\big], \nonumber\\[1ex]
\text{where} \quad a(\sigma^2) &:= \lim_{N \to \infty} N W^{\mathrm{IWD}}(\sigma^2).
\end{align}
where the factorisation uses $\Phi_i \perp \sigma_i^2$.
Because $a(\cdot)$ is strictly decreasing, the Chebyshev
(covariance) inequality gives
$\mathbb{E}[a(\sigma^2)\sigma^2]
 \le \mathbb{E}[a(\sigma^2)]\,\mathbb{E}[\sigma^2]
 = \mathbb{E}[\sigma^2]$
(using $\mathbb{E}[a(\sigma^2)] = 1$ from normalisation).
Hence the IWD numerator is asymptotically bounded by the uniform
numerator $\mathbb{E}[\Phi]\,\mathbb{E}[\sigma^2]$
(which also uses $\Phi_i \perp \sigma_i^2$).

\medskip
\textbf{Step 3: Preserve the denominator.}

By (A2), $\sigma_i^2 \perp (c_i^2, \Phi_i)$, so the weights are
independent of the products $\Phi_i c_i^2$ and
\begin{equation}
\sum_i W_i^{\mathrm{IWD}}\,\Phi_i\,c_i^2
  \;\xrightarrow{p}\; \mathbb{E}\big[\Phi\, c^2\big]
  \;\xleftarrow{p}\;
\sum_i W_i^{\mathrm{unif}}\,\Phi_i\,c_i^2 .
\end{equation}

\begin{table*}[t]
\centering
\setlength{\tabcolsep}{6pt} 

{\small
\begin{tabular}{lllll}
\toprule
\textbf{Task} & \textbf{Training} & \textbf{ID Eval} & \textbf{OOD Sets} & \textbf{Metric} \\
\midrule
NLI (MNLI)        & MNLI train       & MNLI-m dev       & HANS, SNLI                     & Accuracy  \\
Sentiment (SST-2) & SST-2 train      & SST-2 dev        & Yelp-Polarity, CR               & Accuracy  \\
QA (SQuAD v2)     & SQuAD v2 train   & SQuAD v2 dev     & NewsQA,\ Natural Questions      & F1       \\
NER (CoNLL-2003)  & CoNLL-2003 train & CoNLL-2003 test  & WNUT-17, OntoNotes 5.0         & Entity-F1 \\
\bottomrule
\end{tabular}
}

\caption{Evaluation datasets. ID sets share the training distribution; OOD sets target specific distribution shifts. Metrics follow standard task practice.}
\label{tab:datasets}
\end{table*}

\medskip
\textbf{Step 4: Combine to obtain the ratio bound.}

From Step 3, for large $N$, the denominators converge to the same constant $\mathbb{E}[\Phi\, c^2]$. Substituting this into the S2C ratio from~\eqref{eq:rho-scalar-relaxed} and applying the asymptotic numerator bound from Step 2:
\begin{equation}
\rho^2\!\left(W^{\mathrm{IWD}}\right)
  \xrightarrow{p} \frac{\mathbb{E}[\Phi]\,\mathbb{E}[a(\sigma^2)\sigma^2]}
         {\mathbb{E}[\Phi\, c^2]}
  \;\leq\;
  \frac{\mathbb{E}[\Phi]\,\mathbb{E}[\sigma^2]}
       {\mathbb{E}[\Phi\, c^2]}
  \xleftarrow{p} \rho^2\!\left(W^{\mathrm{unif}}\right).
\end{equation}

\medskip
\textbf{Equality condition.}
Equality in the Chebyshev covariance inequality holds if and only if
$\sigma_i^2$ is constant across all training samples (i.e.\ $\sigma_i^2 = \bar{\sigma}^2$ for all $i$), in which case $a(\sigma^2) \equiv 1$.
In that case, both weighting schemes are asymptotically equivalent and
$\rho^2(W^{\mathrm{IWD}}) \xrightarrow{p} \rho^2(W^{\mathrm{unif}})$.

\end{proof}

\medskip
\textbf{Remark (Homogeneous case).}
If $\Phi_i \equiv \Phi$ for all $i$, everything above reduces to the original finite-$N$ argument (the Chebyshev sum inequality applies, so no Law of Large Numbers is needed for the numerator bound). We present this homogeneous case as a special case in Appendix A.3.

\section{B. Experimental Details}
\label{app:experiments}

\subsection{B.1.Dataset Details}
\label{app:datasets}

\paragraph{Natural-language inference (MNLI).}
We train on the MultiNLI matched split~\cite{williams2018mnli} (393k premise--hypothesis pairs, three-way labels) and evaluate ID on its matched validation set. For OOD evaluation, we use HANS~\cite{mccoy2019hans} (30k samples) to test reliance on syntactic shortcuts (lexical overlap, subsequence, and constituent heuristics), and SNLI~\cite{bowman2015snli} (570k image caption samples) to evaluate generalization across a linguistic register shift.

\paragraph{Sentiment Classification (SST-2).}
We train on SST-2 training subsets~\cite{socher2013sst} and evaluate ID on the development set. OOD evaluation uses four binary sentiment benchmarks targeting distinct domain and register shifts: SemEval-2017 Task 4A~\cite{rosenthal2017semeval} (social media), Customer Reviews (CR)~\cite{hu2004cr} (consumer products), Yelp Polarity~\cite{zhang2015character} (business reviews), and Financial PhraseBank (FPB)~\cite{malo2014good} (financial news).

\begin{table*}[t]
\centering
\small
\begin{tabular}{@{} l l p{0.65\linewidth} @{}}
\toprule
\textbf{Task} & \textbf{Environment} & \textbf{Text} \\
\midrule
\multirow{4}{*}{MNLI}
  & \textit{Original} (label = entailment) & P: \textit{``The children were playing soccer in the field.''} \newline H: \textit{``Kids were engaged in a sport.''} \\ 
  & \texttt{negation} & P: \textit{``\ldots{}in the field, \textbf{and no adult was refereeing}.''} \\
  & \texttt{low\_overlap} & H: \textit{``Young people participated in athletic activity.''} \\
  & \texttt{syntactic\_shortcut\_break} & H: \textit{``It was a sport in which the kids were engaged.''} \\
\midrule
\multirow{4}{*}{SST-2}
  & \textit{Original} (label = positive) & \textit{``The film was absolutely wonderful.''} \\
  & \texttt{negation\_invariant} & \textit{``The film was \textbf{not disappointing} at all.''} \\
  & \texttt{social\_media\_register} & \textit{``tbh the film was absolutely wonderful [music] \#mustwatch''} \\ 
  & \texttt{domain\_transposition} & \textit{``The \textbf{product} was absolutely \textbf{excellent}.''} \\
\midrule
\multirow{4}{*}{NER}
  & \textit{Original} & \textit{[Barack Obama]$_{\text{PER}}$ visited [Paris]$_{\text{LOC}}$ .} \\
  & \texttt{lowercase\_transform} & \textit{[barack obama]$_{\text{PER}}$ visited [paris]$_{\text{LOC}}$ .} \\
  & \texttt{entity\_substitution\_same\_type} & \textit{[Amina Watanabe]$_{\text{PER}}$ visited [Zurich]$_{\text{LOC}}$ .} \\
  & \texttt{social\_media\_register} & \textit{omg [barack obama]$_{\text{PER}}$ pulled up to [paris]$_{\text{LOC}}$ [music] \#travel} \\
\bottomrule
\end{tabular}
\caption{Illustrative examples of counterfactual environment generation across three tasks. Each row shows one environment's rewrite, with the modified span emphasised. Bold text in MNLI/SST-2 shows the injected transformation; NER entities are shown in brackets with their type.}
\label{tab:env-examples}
\end{table*}

\paragraph{Extractive Question Answering (SQuAD~v2).}
We train on SQuAD~v2 subsets~\cite{rajpurkar2018squadv2} (including unanswerable questions) and evaluate ID on the standard development split. OOD performance is evaluated on NewsQA~\cite{trischler2017newsqa} for a news-domain shift, and Natural Questions (NQ)~\cite{kwiatkowski2019natural} to test generalization to real Google search queries paired with Wikipedia articles.

\paragraph{Named-entity recognition (CoNLL-2003).}
Models are trained on English CoNLL-2003~\cite{sang2003conll} (4 entity types: PER, ORG, LOC, MISC) and evaluated ID on the test set. OOD evaluation uses WNUT-17~\cite{derczynski2017wnut} (register and entity-distribution shifts) and OntoNotes~5.0~\cite{weischedel2013ontonotes,pradhan2013ontonotes} (genre and label-schema shifts, projecting its 18-type schema down to PER/ORG/LOC). All setups report entity-level micro-F1.
\subsection{B.2.Model Details}
\label{app:models}

\paragraph{DeBERTa family.}
The student is \texttt{deberta-v3-xsmall}~\cite{he2021deberta} (22M parameters).
Teachers are task-specific \texttt{deberta-v3-small} (44M parameters) models fine-tuned on each task's full training set and selected by ID validation performance.
We use the sequence-classification head for MNLI and SST-2, the question-answering head for SQuAD~v2, and the token-classification head for NER.
Student and teacher share the DeBERTa-v3 tokenizer with a maximum sequence length of 256 tokens for classification and NER, and 384 for SQuAD~v2.

\paragraph{Qwen family.}
The student is \texttt{Qwen2.5-0.5B}~\cite{yang2025qwen25} and the teacher is \texttt{Qwen2.5-1.5B}.
We use the sequence-classification head for MNLI and SST-2, the question-answering head for SQuAD~v2, and the token-classification head for NER.
Student and teacher share the Qwen2.5 tokenizer with a maximum sequence length of 256 tokens for classification and NER, and 384 for SQuAD~v2.

\subsection{B.3. Baseline Implementation Details}
\label{app:baselines}

All baselines use the same student architecture, tokenizer,
training schedule, and hardware as \invkd{}. Only the
distillation objective differs.

\paragraph{Vanilla KD~\cite{hinton2015distilling}.}
The student minimizes a weighted sum of the ground-truth
cross-entropy loss and the forward KL divergence between
temperature-scaled teacher and student outputs:
\begin{equation}
\mathcal{L}_{\mathrm{KD}}
  = (1-\alpha)\,\mathcal{L}_{\mathrm{CE}}
  + \alpha\,T^2\,
    \mathrm{KL}\!\bigl(
      \sigma(z_T/T)\,\|\,\sigma(z_S/T)
    \bigr),
\label{eq:vanilla-kd}
\end{equation}
where $z_T$ and $z_S$ are teacher and student logits, $T$ is the
temperature, and $\alpha$ is the KD weight. For SQuAD~v2, the KL
is computed separately for start and end logits. Training uses
the original (non-augmented) data.

\paragraph{RevKL~\cite{gu2023minillm}.}
We replace the forward KL with reverse KL,
$\mathrm{KL}(p_S\|p_T)$, adapting the objective to
classification, extractive QA, and token classification.
Training uses the original (non-augmented) data.

\begin{table*}[t]
\centering
\setlength{\tabcolsep}{6pt}
{\small
\begin{tabular}{l cccccc}
\toprule
\textbf{Task / Dataset} & \textbf{$\tau = 0.1$} & \textbf{$\tau = 0.3$} & \textbf{$\tau = 1.0$} & \textbf{$\tau = 3.0$} & \textbf{$\tau = 10.0$} & \textbf{$\tau = 30.0$} \\
\midrule
\multicolumn{7}{l}{\textbf{NLI (DeBERTa-v3, 5k)}} \\
~~MNLI (ID) & $81.32 \pm 0.27$ & $82.30 \pm 0.33$ & $82.58 \pm 0.31$ & $\mathbf{82.63 \pm 0.42}$ & $82.57 \pm 0.48$ & $82.55 \pm 0.50$ \\
~~HANS (OOD-1) & $52.93 \pm 1.74$ & $53.56 \pm 0.60$ & $53.36 \pm 1.35$ & $\mathbf{54.03 \pm 2.28}$ & $53.98 \pm 1.78$ & $53.91 \pm 1.61$ \\
~~SNLI (OOD-2) & $76.39 \pm 0.87$ & $78.21 \pm 0.51$ & $78.61 \pm 0.29$ & $78.52 \pm 0.55$ & $78.76 \pm 0.08$ & $\mathbf{78.84 \pm 0.20}$ \\
\cmidrule(lr){1-7}
\multicolumn{7}{l}{\textbf{Sentiment Analysis (DeBERTa-v3, 5k)}} \\
~~SST-2 (ID) & $92.20 \pm 0.23$ & $92.13 \pm 0.35$ & $\mathbf{92.28 \pm 0.24}$ & $92.05 \pm 0.33$ & $92.09 \pm 0.23$ & $92.13 \pm 0.07$ \\
~~SemEval (OOD-1) & $88.16 \pm 0.36$ & $88.48 \pm 0.21$ & $88.45 \pm 0.46$ & $\mathbf{88.63 \pm 0.32}$ & $88.62 \pm 0.45$ & $88.61 \pm 0.51$ \\
~~CR (OOD-2) & $88.21 \pm 0.93$ & $\mathbf{88.30 \pm 0.96}$ & $88.21 \pm 0.77$ & $87.77 \pm 0.96$ & $87.85 \pm 0.93$ & $87.94 \pm 1.34$ \\
~~Yelp-Polarity (OOD-3) & $91.60 \pm 0.56$ & $91.67 \pm 0.43$ & $\mathbf{91.70 \pm 0.49}$ & $91.43 \pm 0.53$ & $91.43 \pm 0.52$ & $91.51 \pm 0.55$ \\
~~FPB (OOD-4) & $92.65 \pm 0.16$ & $93.05 \pm 0.66$ & $92.78 \pm 0.67$ & $92.68 \pm 0.75$ & $\mathbf{93.24 \pm 0.43}$ & $92.93 \pm 0.36$ \\
\cmidrule(lr){1-7}
\multicolumn{7}{l}{\textbf{Extractive QA (DeBERTa-v3, 5k)}} \\
~~SQuAD-v2 (ID) & $\mathbf{47.64 \pm 2.72}$ & $47.14 \pm 2.60$ & $47.51 \pm 1.27$ & $46.75 \pm 0.75$ & $45.80 \pm 0.87$ & $45.36 \pm 0.74$ \\
~~NewsQA (OOD-1) & $\mathbf{29.20 \pm 2.00}$ & $27.56 \pm 1.97$ & $24.82 \pm 0.41$ & $23.00 \pm 1.61$ & $22.95 \pm 0.91 $ & $21.57 \pm 1.03$ \\
~~Natural Questions (OOD-2) & $\mathbf{45.32 \pm 1.47}$ & $42.99 \pm 2.68$ & $35.87 \pm 0.29$ & $33.37 \pm 2.19$ & $38.32 \pm 1.87$ & $35.49 \pm 2.47$ \\
\cmidrule(lr){1-7}
\multicolumn{7}{l}{\textbf{NER (DeBERTa-v3, 5k)}} \\
~~CoNLL-2003 (ID) & $91.42 \pm 0.08$ & $\mathbf{91.72 \pm 0.23}$ & $91.62 \pm 0.22$ & $91.50 \pm 0.49$ & $91.06 \pm 0.24$ & $91.05 \pm 0.25$ \\
~~WNUT-17 (OOD-1) & $45.54 \pm 0.65$ & $45.89 \pm 0.18$ & $\mathbf{46.04 \pm 1.11}$ & $45.26 \pm 2.19$ & $43.21 \pm 0.84$ & $43.25 \pm 0.76$ \\
~~OntoNotes 5.0 (OOD-2) & $78.51 \pm 0.25$ & $\mathbf{78.60 \pm 0.26}$ & $78.56 \pm 0.30$ & $78.13 \pm 0.93$ & $77.01 \pm 0.64$ & $77.00 \pm 0.58$ \\
\bottomrule
\end{tabular}
}
\caption{Sensitivity analysis evaluating the impact of the distillation temperature parameter $\tau$ on IWD across four task families using DeBERTa-v3 (trained on 5k subsamples).}
\label{tab:ablation_tau_sensitivity}
\end{table*}

\paragraph{DKD~\cite{zhao2022dkd}.}
Decoupled KD separates the KD loss into
target-class (TCKD) and non-target-class (NCKD) terms:
\begin{equation}
\mathcal{L}_{\mathrm{DKD}}
  = \beta_{\mathrm{t}}\,\mathcal{L}_{\mathrm{TCKD}}
  + \beta_{\mathrm{nt}}\,\mathcal{L}_{\mathrm{NCKD}}.
\label{eq:dkd}
\end{equation}
We use the original settings
($\beta_{\mathrm{t}}=1.0$, $\beta_{\mathrm{nt}}=1.0$) and adapt
the objective to all tasks.

\paragraph{LWD.}
Layer-Wise Distillation augments output-level KD with MSE
matching between corresponding hidden representations.
For DeBERTa, we align the first, middle, and last encoder
layers; for Qwen, the first, third, and last decoder layers.

\paragraph{AugKD.}
Vanilla KD trained on the same augmented dataset as IWD,
treating all $K$ paraphrases as independent training samples
with equal weight. This isolates the effect of invariance
weighting from data augmentation alone.
\subsection{B.4. Training and Hyperparameter Details}
\label{app:training}

\paragraph{Optimiser and schedule.}
All students use AdamW~\cite{loshchilov2019adamw}
($\beta_1=0.9$, $\beta_2=0.999$, $\epsilon=10^{-8}$, weight decay
$0.0$) with 200 warmup steps and linear decay.
DeBERTa students train for 3 epochs (5 for NER) with learning rate
$7\!\times\!10^{-5}$, batch size 32, and fp16 precision.
Qwen students train for 3 epochs (5 for NER) with learning rate
$2\!\times\!10^{-5}$, batch size 8, gradient accumulation of 2
(effective batch size 16), gradient checkpointing, and bf16 precision.

\paragraph{KD hyperparameters.}
All methods use temperature $T=2.0$.
The KD weight $\alpha$ is 5 for DeBERTa and 1 for Qwen, selected
from $\{0.5,1,2,5\}$ on MNLI ID validation and fixed for all tasks.

\paragraph{IWD hyperparameters.}
We use $\tau=1.0$ with batch-mean weight normalization
($\sum_{i\in\mathcal{B}}\tilde{w}_i=|\mathcal{B}|$).
The disagreement metric is Jensen--Shannon divergence for MNLI and
SST-2, one-minus SQuAD-F1 for SQuAD~v2, and one-minus entity-set F1
for NER. The results in the major result table \ref{tab:main_results} for NLI, NER and Extractive QA tasks were acheived by these hyperparameters while for better results in sentiment analysis task we had to use $\tau=0.1$ with SymKL as the distance measure of teacher invariance.

\paragraph{Per-task disagreement metric $D_i$.}
$D_i$ averages pairwise disagreement over the
$\binom{K}{2}$ variant pairs.
For MNLI and SST-2, we compute Jensen--Shannon divergence over
teacher class probabilities (3-class and 2-class softmax,
respectively). For SQuAD~v2 and NER, we use one-minus F1 between
teacher-decoded answer strings and entity mention sets,
respectively.

\paragraph{Counterfactual data generation.}
Variants are generated offline using GPT-4.1-mini with prompts
preserving task labels (sentiment, entailment, answer spans, or
entity annotations) while varying superficial features such as
register, syntax, and vocabulary.
For NER, entity surface forms are preserved while context is
paraphrased.
Each sample has $K\in[2,6]$ variants (mean $\approx5$); samples
with $K<2$ or label-inconsistent variants are removed.

\paragraph{Hardware and seeds.}
Experiments use 2 NVIDIA RTX A4500 GPUs (20\,GB each).
The Qwen SQuAD~v2 teacher uses batch size 4 with gradient
accumulation of 8.
Each configuration is evaluated over five seeds
$\{42,43,44,45,46\}$, reporting mean and standard deviation.

\subsection{B.5. Counterfactual Environment Types}
\label{app:envs}

This subsection lists the paraphrase environments used to generate
counterfactual variants. Each environment targets a distinct
spurious surface feature while preserving labels. All environments
are applied to training samples, and label-inconsistent variants
are discarded. The number of variants varies across experiments. Table \ref{tab:env-examples} shows several examples of environments data generation.

\paragraph{MNLI (7 environments).}
We generate variants by perturbing lexical overlap
(\texttt{high/low\_overlap}), negation patterns
(\texttt{negation}), irrelevant context
(\texttt{distractor}), syntactic heuristics
(\texttt{syntactic\_shortcut\_break}), length bias
(\texttt{length\_decorrelation}), and voice structure
(\texttt{voice\_alternation}) while preserving labels.

\paragraph{SST-2 (8 environments).}
Variants perturb sentiment-related shortcuts through
negation invariance, lexical overlap
(\texttt{lexical\_paraphrase\_low\_overlap}), domain and register
shifts (\texttt{domain\_transposition},
\texttt{social\_media\_register}), sentiment cue removal,
intensifier/hedge changes, and comparative reframing while
preserving polarity.

\paragraph{SQuAD~v2 (6 environments).}
Variants target QA shortcuts through distractor insertion,
question paraphrasing, answer-position shifts, decoy answer spans,
entity substitution, and voice alternation while preserving the
gold answer.

\paragraph{NER (7 environments).}
Variants include 3 programmatic transforms
(lowercasing, typographical noise, and punctuation perturbation)
and four LLM-based transforms targeting entity substitution,
emerging entities, social-media register, and syntax variation.
Entity annotations are preserved for all.

\begin{table*}[t]
\centering
\small
\setlength{\tabcolsep}{16pt} 
\begin{tabular}{l c cc}
\toprule
\textbf{Distance Metric} & \textbf{ID} & \textbf{OOD 1} & \textbf{OOD 2} \\
\midrule
\multicolumn{4}{l}{\textbf{MNLI Benchmark}} \\
& \textit{MNLI-m} & \textit{HANS} & \textit{SNLI} \\
\cmidrule(lr){2-2} \cmidrule(lr){3-3} \cmidrule(lr){4-4}
Jensen--Shannon ($D_{\text{JS}}$) & \textbf{82.60 $\pm$ 0.28} & \textbf{53.35 $\pm$ 1.38} & \textbf{78.68 $\pm$ 0.34} \\
Symmetric KL ($D_{\text{SymKL}}$) & 81.38 $\pm$ 0.40 & 53.01 $\pm$ 2.12 & 77.43 $\pm$ 0.74 \\
$L_2$ Distance & 82.38 $\pm$ 0.27 & 52.76 $\pm$ 1.03 & 78.58 $\pm$ 0.33 \\
\midrule
\multicolumn{4}{l}{\textbf{SST-2 Sentiment Benchmark}} \\
& \textit{SST-2} & \textit{CR} & \textit{Yelp} \\
\cmidrule(lr){2-2} \cmidrule(lr){3-3} \cmidrule(lr){4-4}
Jensen--Shannon ($D_{\text{JS}}$) & \textbf{92.28 $\pm$ 0.24} & 88.21 $\pm$ 0.77 & \textbf{91.70 $\pm$ 0.49} \\
Symmetric KL ($D_{\text{SymKL}}$) & 92.01 $\pm$ 0.54 & 87.85 $\pm$ 1.63 & 91.73 $\pm$ 0.44 \\
$L_2$ Distance & 92.01 $\pm$ 0.37 & \textbf{88.30 $\pm$ 1.48} & 91.55 $\pm$ 0.55 \\
\midrule
\multicolumn{4}{l}{\textbf{Extractive QA Benchmark}} \\
& \textit{SQuAD-v2} & \textit{NewsQA} & \textit{NQ} \\
\cmidrule(lr){2-2} \cmidrule(lr){3-3} \cmidrule(lr){4-4}
Span F1 & 26.60 $\pm$ 0.70 & 25.21 $\pm$ 0.22 & 37.35 $\pm$ 0.96 \\
Jaccard Distance & 27.03 $\pm$ 1.13 & \textbf{25.36 $\pm$ 0.23} & 37.78 $\pm$ 0.29 \\
Levenshtein Distance & \textbf{27.10 $\pm$ 1.51} & 24.65 $\pm$ 0.23 & \textbf{38.19 $\pm$ 1.38} \\
\midrule
\multicolumn{4}{l}{\textbf{NER Benchmark}} \\
& \textit{CoNLL-03} & \textit{WNUT17} & \textit{OntoNotes} \\
\cmidrule(lr){2-2} \cmidrule(lr){3-3} \cmidrule(lr){4-4}
Entity Set F1 & \textbf{91.62 $\pm$ 0.22} & \textbf{46.04 $\pm$ 1.11} & \textbf{78.56 $\pm$ 0.30} \\
Entity Set Jaccard & 91.62 $\pm$ 0.33 & 45.80 $\pm$ 0.57 & 78.46 $\pm$ 0.37 \\
Entity Flip Rate & 91.62 $\pm$ 0.33 & 45.80 $\pm$ 0.57 & 78.46 $\pm$ 0.37 \\
\bottomrule
\end{tabular}
\caption{Sensitivity of IWD to distance metric formulations across tasks on DeBERTa-v3 ($N=5\text{k}$).}
\label{tab:distance_metrics_ablation}
\end{table*}

\begin{table}[t]
\centering
\setlength{\tabcolsep}{2.5pt}

{\small
\begin{tabular}{l cccc}
\toprule
\textbf{Variant Strategy} & \textbf{MNLI-M} & \textbf{MNLI-MM} & \textbf{HANS} & \textbf{SNLI} \\
\textbf{\& Metric} & \textbf{(ID)} & \textbf{(ID)} & \textbf{(OOD1)} & \textbf{(OOD2)} \\
\midrule
\multicolumn{5}{l}{\textbf{JS Distance (Primary Metric)}} \\
~~All Variants & \textbf{84.238} & \textbf{84.947} & \textbf{57.323} & \textbf{80.542} \\
~~Original Variant & 82.812 & 83.411 & 56.206 & 77.433 \\
~~Random Variant & 82.873 & 83.441 & 53.103 & 78.185 \\
\cmidrule(lr){1-5}
\multicolumn{5}{l}{\textbf{Sym KL Divergence}} \\
~~All Variants & 83.311 & 84.121 & 54.420 & 78.784 \\
~~Original Variant & 82.170 & 83.197 & 52.076 & 75.706 \\
~~Random Variant & 82.292 & 82.994 & 52.643 & 77.027 \\
\cmidrule(lr){1-5}
\multicolumn{5}{l}{\textbf{$L_2$ Distance}} \\
~~All Variants & \underline{84.014} & \underline{84.875} & 55.036 & \underline{79.617} \\
~~Original Variant & 82.934 & 83.868 & \underline{55.626} & 77.037 \\
~~Random Variant & 82.638 & 83.390 & 52.923 & 78.195 \\
\bottomrule
\end{tabular}
} 

\caption{Ablation study evaluating environment variant selection strategies during student distillation across different distance metrics using DeBERTa-v3 on ID and OOD benchmarks.}
\label{tab:ablation_distance_variants}
\end{table}

\section{C. Ablation Experiments and Results}
\label{sec:appendix_ablations}

\paragraph{Sensitivity to Weight Temperature ($\tau$).}
Table~\ref{tab:ablation_tau_sensitivity} evaluates the weighting temperature $\tau \in \{0.1, 0.3, 1.0, 3.0, 10.0, 30.0\}$, which dictates the sharpness of sample re-weighting relative to teacher divergence ($D_i$). Classification tasks (MNLI, SST-2) and NER are robust across $\tau \in [0.3, 3.0]$, with ID and mean OOD metrics varying by less than $0.7$ and $1.5$ points, respectively. In contrast, extractive QA (SQuAD v2) is highly sensitive due to high variance in teacher answer-span fluctuations, benefiting dramatically from aggressive downweighting ($\tau = 0.1$). Lowering $\tau$ from $30.0$ to $0.1$ yields substantial OOD gains of $+7.6$ F1 on NewsQA ($29.20$) and $+9.8$ F1 on Natural Questions ($45.32$). Thus, while our default setting of $\tau = 1.0$ is optimal for classification and sequence tagging, tasks with high divergence variance like QA benefit from smaller temperatures ($\tau \in \{0.1, 0.3\}$).

\paragraph{Sensitivity to Distance Metric Selection.}
The optimal teacher invariance distance metric depends directly on the structure of the output space. For classification tasks (MNLI, SST-2), bounded and symmetric Jensen--Shannon divergence ($D_{\text{JS}}$) achieves peak performance by preventing extreme target probability tails from overly skewing sample weights compared to Symmetric KL or $L_2$ distance. For structured prediction, metrics directly targeting output boundary shifts yield superior signals: Levenshtein distance performs best for Extractive QA ($38.19$ F1 on NQ), while Entity Set F1 achieves top accuracy for NER ($46.04$ F1 on WNUT17). Aligning the metric to the output modality—probability divergences for classification and structural metrics for span extraction—maximizes IWD's ability to isolate non-causal prediction variance.

\paragraph{Environment Variant Utilization.}
Table~\ref{tab:ablation_distance_variants} compares three student training presentation strategies using teacher divergence weights ($D_i$) computed over the full set of variants: (1) \textit{All Variants} (distilling on every generated environment state), (2) \textit{Original Variant} (distilling on clean inputs only), and (3) \textit{Random Variant} (selecting a single random variant per sample). Distilling across \textbf{All Variants} consistently achieves top performance across all benchmarks. Restricting student inputs to the \textit{Original} or \textit{Random} variant causes substantial drops (e.g., losing up to $4.2$ percentage points on HANS). Evaluating across multiple distance measures ($JS$, Symmetric $KL$, $L_2$) confirms this trend is metric-agnostic, with \textbf{JS Distance} providing the most effective weighting signal overall.

\paragraph{Weight Normalisation.}
We compare default batch-mean normalisation ($\tilde{w}_i = |\mathcal{B}| \cdot w_i / \sum_{j\in\mathcal{B}} w_j$) against unnormalised weighting across MNLI and SST-2 (Table~\ref{tab:ablation_normalization}). Performance is effectively equivalent: on MNLI, ID accuracy shifts by $<0.17$ points and mean OOD performance differs by $\le 0.07$ points across DeBERTa-v3 and Qwen-2.5. On SST-2, unnormalised weighting shows minor average OOD gains ($+0.63$ on DeBERTa-v3, $+0.41$ on Qwen-2.5) but fluctuates across individual datasets (e.g., $+1.32$ on FPB vs. $-0.22$ on SemEval-2017 for DeBERTa-v3), reflecting single-seed variance. We retain batch-mean normalisation because it: (i) decouples the effective learning rate from the absolute scale of $D_i$, stabilizing $\tau$ tuning across tasks; and (ii) strictly preserves the per-sample relative importance weights required by our theoretical bound.

\begin{table}[t]
\centering
\setlength{\tabcolsep}{3pt}
{\small
\begin{tabular}{ll ccc}
\toprule
\textbf{Backbone} & \textbf{Norm} & \textbf{ID} & \textbf{HANS (OOD1)} & \textbf{SNLI (OOD2)} \\
\midrule
\multicolumn{5}{l}{\textbf{Task: MNLI}} \\
DeBERTa-v3 & Batch-Mean & \textbf{83.25} & \textbf{51.98} & 79.66 \\
           & None       & 83.22 & 51.80 & \textbf{79.91} \\
           & $\Delta$   & -0.03 & -0.18 & +0.25 \\
\cmidrule(lr){1-5}
Qwen-2.5   & Batch-Mean & \textbf{81.87} & 53.19 & \textbf{78.39} \\
           & None       & 81.70 & \textbf{53.20} & 78.23 \\
           & $\Delta$   & -0.17 & +0.02 & -0.16 \\
\midrule
\end{tabular}

\vspace{2pt}

\begin{tabular}{ll ccccc}
\textbf{Backbone} & \textbf{Norm} & \textbf{ID} & \textbf{SemEval} & \textbf{CR} & \textbf{Yelp} & \textbf{FPB} \\
\midrule
\multicolumn{7}{l}{\textbf{Task: SST-2}} \\
DeBERTa-v3 & Batch-Mean & 91.40 & \textbf{89.27} & 85.64 & 91.38 & 90.19 \\
           & None       & \textbf{91.74} & 89.05 & \textbf{86.70} & \textbf{91.74} & \textbf{91.51} \\
           & $\Delta$   & +0.34 & -0.22 & +1.06 & +0.36 & +1.32 \\
\cmidrule(lr){1-7}
Qwen-2.5   & Batch-Mean & 92.89 & \textbf{74.19} & \textbf{73.40} & 78.45 & 83.58 \\
           & None       & \textbf{93.00} & 73.91 & 72.34 & \textbf{79.13} & \textbf{85.87} \\
           & $\Delta$   & +0.11 & -0.28 & -1.06 & +0.68 & +2.29 \\
\bottomrule
\end{tabular}
}
\caption{Ablation study comparing batch-mean normalization against no weight normalization on MNLI and SST-2 across DeBERTa-v3 and Qwen-2.5 backbones.}
\label{tab:ablation_normalization}
\end{table}

\begin{table}[t]
\centering
\setlength{\tabcolsep}{2.5pt}

{\small
\begin{tabular}{@{}l c c c c@{}}
\toprule
\textbf{Task} & $\boldsymbol{w_i}$ & \textbf{\% }$\boldsymbol{w_i}$ & \textbf{\% }$\boldsymbol{w_i}$ & \textbf{Effective} \\
              & \textbf{(Mean $\pm$ Std)} & $\mathbf{\ge 0.9}$ & $\mathbf{< 0.5}$ & $\boldsymbol{N}$ \\
\midrule
MNLI     & $0.87 \pm 0.10$ & 40.8\% & 0.0\%  & 4,931 (98.6\%) \\
SST-2    & $0.95 \pm 0.08$ & 78.0\% & 0.0\%  & 4,966 (99.3\%) \\
SQuAD v2 & $0.86 \pm 0.20$ & 63.5\% & 4.4\%  & 4,748 (95.0\%) \\
NER      & $0.48 \pm 0.06$ & 0.0\%  & 58.7\% & 4,931 (98.6\%) \\
\bottomrule
\end{tabular}
} 

\caption{Distribution of per-sample IWD weights across training data at $\tau=1.0$ and batch-mean normalization ($N=5{,}000$ per task). Effective sample size $(\sum_i w_i)^2 / \sum_i w_i^2$ measures information retention relative to uniform weighting, where values near $N$ indicate gentle re-weighting rather than aggressive filtering.}
\label{tab:weight_distribution}
\end{table}

\paragraph{Weight Distribution Across Tasks.}
Table~\ref{tab:weight_distribution} reports IWD weight statistics across $N=5{,}000$ samples at $\tau=1.0$. High effective sample sizes across all tasks ($95.0\%$--$99.3\%$) confirm that IWD performs smooth continuous re-weighting rather than aggressive sample filtering. Weight profiles naturally adapt to task complexity: classification tasks (SST-2, MNLI) retain high mean weights ($0.95$ and $0.87$), whereas SQuAD v2 exhibits higher variance ($\pm 0.20$) and $4.4\%$ low-weight samples ($w_i < 0.5$) due to teacher sensitivity on complex extractive spans. On NER, higher overall perturbation sensitivity yields a lower mean weight ($0.48 \pm 0.06$), yet a $98.6\%$ effective sample size confirms that relative importance ordering is preserved while maintaining overall data capacity.

\end{document}